\documentclass{article}

\usepackage[final,nonatbib]{neurips_2024}

\usepackage{amsmath}
\usepackage{amsthm}
\usepackage{amssymb}
\usepackage{multirow}
\usepackage{float}

\newtheorem{thm}{Theorem}

\newtheorem{lem}{Lemma}

\newtheorem{ass}{Assumption}
\newtheorem{claim}{Claim}
\usepackage[
  style=numeric,
  sorting=none,
  defernumbers=true,
  backref=true
]{biblatex}
\usepackage[colorlinks=true, citecolor=blue, linkcolor=blue, urlcolor=blue, plainpages=false]{hyperref}
\usepackage[normalem]{ulem}

\DefineBibliographyStrings{english}{%
    backrefpage = {page},
    backrefpages = {pages},
}

\renewbibmacro*{pageref}{%
    \iflistundef{pageref}
        {}
        {\printtext[parens]{\printlist[pageref][-\value{listtotal}]{pageref}}}%
}

\usepackage[T1]{fontenc}    
\usepackage{url}            
\usepackage{booktabs}       
\usepackage{amsfonts}       
\usepackage{nicefrac}       
\usepackage{microtype}      
\usepackage{xcolor}         
\definecolor{bg}{RGB}{255,249,227}

\usepackage{listings}

\definecolor{codegreen}{rgb}{0,0.6,0}
\definecolor{codegray}{rgb}{0.5,0.5,0.5}
\definecolor{codepurple}{rgb}{0.58,0,0.82}

\lstdefinestyle{mystyle}{
    backgroundcolor=\color{white},   
    commentstyle=\color{codegreen},
    keywordstyle=\color{magenta},
    numberstyle=\tiny\color{codegray},
    stringstyle=\color{codepurple},
    basicstyle=\ttfamily\footnotesize,
    breakatwhitespace=false,         
    breaklines=true,                 
    captionpos=b,                    
    keepspaces=true,                 
    numbers=left,                    
    numbersep=5pt,                  
    showspaces=false,                
    showstringspaces=false,
    showtabs=false,                  
    tabsize=2
}

\usepackage{booktabs}
\usepackage{array}              

\usepackage[colorinlistoftodos,bordercolor=bg,backgroundcolor=bg,linecolor=bg]{todonotes}

\title{Analysis of Schedule-Free Nonconvex Optimization}

\author{
  Connor Brown\\
  Department of Electrical and Computer Engineering\\
  Princeton University\\
  \texttt{connorbrown@princeton.edu}
}

\begin{document}
\maketitle
\begin{abstract}
First-order methods underpin most large-scale learning algorithms, yet their classical convergence guarantees hinge on carefully scheduled step-sizes that depend on the total horizon \(T\), which is rarely known in advance.  The Schedule-Free (SF) method \cite{defazio2024roadless} promises optimal performance with hyperparameters that are independent of \(T\) by interpolating between Polyak–Ruppert averaging and momentum, but nonconvex analysis of SF has been limited or reliant on strong global assumptions.  We introduce a robust Lyapunov framework that, under only \(L\)-smoothness and lower-boundedness, reduces SF analysis to a single-step descent inequality.  This yields horizon-agnostic bounds in the nonconvex setting: \(O(1/\log T)\) for constant step + PR averaging, \(O(\log T/T)\) for a linearly growing step-size, and a continuum of \(O\bigl(T^{-(1-\alpha)}\bigr)\) rates for polynomial averaging.  We complement these proofs with Performance Estimation Problem (PEP) experiments \cite{drori2014performance} that numerically validate our rates and suggest that our \(O(1/\log T)\) bound on the original nonconvex SF algorithm may tighten to \(O(1/T)\). Our work extends SF’s horizon-free guarantees to smooth nonconvex optimization and charts future directions for optimal nonconvex rates.
\end{abstract}

\section{Introduction} 
    First-order methods remain the workhorses of modern machine-learning pipelines because each step costs just one gradient while delivering competitive wall-clock performance. Classical theory, however, ties their convergence rates to carefully scheduled stepsizes that depend on the total training horizon \(T\). Gradient descent with stepsize \(\eta_t\propto1/\sqrt{T} \) attains the optimal \(f(x_T)-f^* = \mathcal{O}(1/\sqrt{T})\) rate on smooth convex objectives \cite{nemirovski2009robust} and \(\min_{0\le t < T}||\nabla f(x_t)||^2==\mathcal{O}(1/\sqrt{T})\) rate for nonconvex objectives \cite{ghadimi2013stochastic, lei2019variance}. Yet in practice \(T\) is rarely known in advance --- training may stop early for validation, resume for fine-tuning, or continue on a new dataset --- so practitioners improvise with piece-wise decay, cosine annealing, or other heuristic schedules whose theoretical footing is shallow.

    Schedule-free algorithms seek to break this dependence entirely. Defazio et al. (2024) \cite{defazio2024roadless} introduced a three-sequence method, Schedule-Free, that smoothly interpolates between Polyak-Ruppert~\cite{polyak1990, Polyak1992, ruppert1988efficient} averaging and stochastic gradient descent with momentum. Their analysis recovers the best known horizon-free rates in convex and strongly convex settings, and extensive experiments suggest that the same hyperparameters work well on deep neural networks. However, nonconvex analysis was left for future work. Ahn et al. (2024) \cite{ahn2024generalframeworkonlinetononconvexconversion} filled part of that gap by proving convergence under strong global assumptions (Lipschitz gradients, well-behaved property, vanishing variance), but those assumptions exclude many applications and do not cover most mainstream models.

    Our contribution is a Lyapunov-style framework which reduces convergence analysis of the three-sequence Schedule-Free algorithm to a single-step descent inequality with minimal deterministic assumptions, i.e., smoothness and function lower-boundedness. In doing so, we upper-bound the algorithm's performance in nonconvex smooth settings under a broad class of horizon-agnostic hyperparameter values. We also utilize the increasingly popular Performance Estimation Problem (PEP) framework as a powerful method for empirically validating our mathematical results on such problem classes, justifying additional boundedness assumptions between Schedule-Free's iterate sequences, and presenting potential future directions for related work. We note that, because most immediately related results are typically presented in the deterministic regime, we state our headline theorems under zero-noise dynamics. Readers interested in the more general stochastic case can find an extended proof in Appendix C and Appendix D, where we show that --- under additional assumptions which are standard in stochastic optimization (Assumption \ref{ass:assumption-3})--- our deterministic rates remain unchanged up to an additional constant noise factor.
    
\section{Preliminaries and Notation}
    Throughout this paper, we will use \(||\cdot||\) for vector \(\ell_2\)-norm. Let \(f^*:=\inf_{y\in\mathbb{R}^d}f(x)\).
    \begin{def}
    \label{def:l-smoothness}
    We say that \(f:\mathbb{R}^d\rightarrow\mathbb{R}\) is \(L\)-smooth with \(L\ge0\) if it is differentiable and satisfies
    \[
    \begin{aligned}
        &f(u)\le f(v)+\langle \nabla f(v), u-v\rangle + \frac{L}{2}||u-v||^2,\\
        &||\nabla f(u)-\nabla f(v)||^2 \le L^2 ||u-v||^2\quad\quad\quad\quad\quad\quad\forall u,v\in\mathbb{R}^d. 
    \end{aligned}
    \]
    \end{def}
    
    The following assumption is effective throughout and is a standard basis in deterministic nonconvex optimization.
    
    \begin{ass}
    \label{ass:assumption-1}\hfill
    \begin{enumerate} 
        \item \textbf{Smoothness:} The objective function \(f:\mathbb{R}^n\rightarrow\mathbb{R}\) is \(L\)-smooth.
        \item \textbf{Lower-bounded objective:} The objective \(f\) is lower bounded by a finite constant, i.e., \(f^*\) is the well-defined optimal minimum value over \(f\).
    \end{enumerate}
    \end{ass}
    
    Unless otherwise indicated, all proofs are collected in the appendix. We first establish each result in the general stochastic regime and recover the deterministic statements in our theorems by letting the noise variance equal zero. Appendix C lists the additional, standard assumptions needed for that stochastic analysis.
    
\section{Related Work}
    In this section we first introduce the Schedule-Free update rule purely as a mechanistic bridge between averaging and momentum. We then review each parent method in more depth, and finally return to Schedule-Free to summarize the convergence results that have appeared so far and to situate our own contributions.
    
    \subsection{Overview of Schedule-Free}
        The general Schedule-Free (SF) method maintains three sequences with the following updates:\\
        \begin{equation}
        \begin{aligned}
        \label{eq:sf-updates}
            y_t &= (1 - \beta_t) z_t + \beta_t x_t,\\
            z_{t+1} &= z_t - \eta_t\nabla f(y_t,\zeta_t),\\
            x_{t+1}&=(1-c_{t+1}) x_t+c_{t+1} z_{t+1},
        \end{aligned}
        \end{equation}\\
        where \(x_0 = z_0\) and \(\beta_t \in [0,1]\) controls the interpolation between Stochastic Gradient Descent with Momentum (SGD+M) and Polyak-Ruppert averaging. Specifically, the intuition behind SF can be expressed through the following special cases: \(\beta_t=1,\,c_{t+1}=1/(t+1)\rightarrow\) Primal Averaging (equivalent to SGD+M (Theorem \ref{thm:thm-1}) and Stochastic Heavy-Ball Momentum (Theorem \ref{thm:SHBM-spa-correspondence})) and for \(\beta_t=0,\,c_{t+1}=1/(t+1)\rightarrow\) Polyak-Ruppert (arithmetic iterate averaging). By interpolating between momentum and arithmetic averaging, SF seeks to combine the acceleration effects of momentum on bias decay with the variance reduction properties of averaging. 

    \subsection{Polyak-Ruppert averaging}
        Polyak–Ruppert (PR) averaging employs a simple arithmetic average of iterates produced by SGD. Given iterates \(z_{t+1} = z_t - \eta \nabla f(z_t,\zeta_t)\), the averaged solution is defined as:
        \[
            \bar{z}_T = \frac{1}{T} \sum_{t=1}^T z_t,
        \]
        which equivalently can be represented through recursive updates:\\
        \[
        \begin{aligned}
            z_{t+1} &= z_t - \eta_t \nabla f(z_t,\zeta_t)\\
            x_{t+1} &= (1-c_{t+1})x_t + c_{t+1}z_{t+1},
        \end{aligned}
        \]\\
        where choosing \( c_{t+1} = 1/(t+1)\), as in the original SF method, yields the classical arithmetic average. Non‐asymptotic analyses show that PR averaging smooths out gradient noise --- achieving \(f(\bar{x}_T)-f^* = O(1/\sqrt{T})\) in convex settings and \(\mathcal{O}(1/T)\) under strong convexity --- without requiring a vanishing stepsize \cite{bach2011non,rakhlin2012making}.  In practice, these effects translate to reduced sensitivity to stepsize mis‐specification, improved training stability, and enhanced generalization in large‐scale machine learning.
         
        In the case of nonconvex problems, PR averaging has a \(\mathcal{O}(1/\sqrt{T})\) convergence rate to a stationary point; but for both convex and nonconvex problems satisfying additional assumptions (e.g., the Kurdyka–Łojasiewicz (KL) Condition), a \(\min_{0\le t < T}||\nabla f(x_t)||^2=\mathcal{O}(1/T)\) rate can be achieved \cite{gadat2017optimal}. 

    \subsection{SGD with Momentum}
        Stochastic Gradient Descent with Momentum (SGD+M) maintains an auxiliary velocity buffer \(m_{t+1}= \lambda_t m_t + \nabla f(x_t,\zeta_t)\) and updates the iterate by \(x_{t+1}=x_t-\alpha_t m_{t+1}\). Defazio et al.\ and Sebbouh et al.\ independently observed that exactly the same \(\{x_t\}\) sequence can be generated by a Stochastic Primal Averaging (SPA) scheme that is algebraically more convenient for analysis:
        \[
        \begin{aligned}
            z_{t+1}&=z_t-\eta_t\nabla f(x_t,\zeta_t),\\
            x_{t+1}&=(1-c_{t+1})x_t+c_{t+1}z_{t+1},
        \end{aligned}
        \]
        
        which is equivalent to SF when \(\beta_t=1\). The exact correspondence between SPA and SGD+M is summarized below. \\
        
        \begin{thm}[SPA $\leftrightarrow$ SGD+M, \cite{defazio2021}]
        \label{thm:thm-1}
            Assume $z_0=x_0$.  The iterates $\{x_t\}$ produced by SPA and SGD+M are
            identical iff for every $t\ge0$
            \[
                \alpha_t = \eta_t c_{t+1},\qquad
                \lambda_t = 1-c_{t+1}.
            \]
            Conversely, given any $\{\alpha_t,\lambda_t\}$ with $0\!<\!\lambda_t\!<\!1$, choose  $c_{t+1}=\lambda_t$ and $\eta_t=\alpha_t/\lambda_t$ to retrieve the SPA parameters.
        \end{thm}

        If we consider these parameter identities, then the SPA iterates coincide with those of SGD+M for all \(t\ge0\). Because the \(z_t\) sequence appears naturally in Lyapunov arguments (and is expressed directly in SF) we conduct our nonconvex analysis using the SPA form, noting that every result transfers verbatim to momentum via this mapping.
        
        The existing theory establishes three benchmark rates for SPA/SGD+M. When \(f\) is strongly convex, \(f(x_T)-f^\star=O\!(1/T)\). For purely convex \(L\)-smooth objectives, the minimax optimal \(f(x_T)-f^\star=O\!(1/\sqrt{T})\) is achieved. In the nonconvex smooth setting, Defazio et al. \cite{defazio2021} showed that constant hyperparameters are already sufficient for \(\min_{0\le t<T}||\nabla f(x_t)||^2=O\!\bigl(1/\sqrt{T}\bigr)\), matching the best known stochastic rate without requiring a vanishing schedule. For gradually geometrically-decaying \(c_{t+1}\) and \(\eta_t\), the same \(\mathcal{O}(1/\sqrt{T})\) nonconvex smooth rate is also achieved. 
        
        Similarly, Liu et al. also provide optimal convergence rates for SGD+M, conditioned on gradually changing hyperparameters \cite{liu2020improvedanalysisstochasticgradient}. In fact, for the case described in Theorem \ref{thm:const-step} of this paper, \(c_{t+1}=1/(t+1)\) is considered to decay "too fast" under both Defazio and Liu's analyses; and while Defazio's framework fails to demonstrate descent in this case, Liu's analysis concludes that \(\min_{0\le t < T}||\nabla f(x_t)||^2=\mathcal{O}(1/T)\), as similarly arrived at in our Theorem \ref{thm:const-step}.
        
    \subsection{Stochastic Heavy-Ball Momentum}
        The stochastic Heavy-Ball method updates the current point by adding a scaled gradient step and a momentum term that re-uses the last displacement:
        \[
        x_{t+1}=x_t-\lambda_t\,\nabla f(x_t,\zeta_t)+\theta_t\bigl(x_t-x_{t-1}\bigr),
        \]
        where \(\lambda_t>0\) is the stepsize and \(\theta_t\in[0,1)\) is the momentum weight; both are usually kept constant in practice. Heavy-Ball can be written in the SPA/averaging form used by
        SF, so every guarantee we derive for SPA carries over to
        Heavy-Ball through the following algebraic map.\\
        
        \begin{thm}[SHBM \(\leftrightarrow\) SPA]
        \label{thm:SHBM-spa-correspondence}
        Assume \(z_0=x_0\).  The iterates \(\{x_t\}\) produced by
        Stochastic Heavy-Ball Momentum (SHBM) and by SPA coincide for all
        \(t\ge0\) if and only if
        \[
            \eta_t=\frac{\lambda_t\bigl(1-c_t+\theta_t c_t\bigr)}{\theta_t 
            c_t},\qquad c_{t+1}= \frac{\theta_t c_t}{1-c_t+\theta_t c_t}.
        \]
        Conversely, given SPA parameters \(\{\eta_t,c_{t+1}\}\) the Heavy-Ball coefficients can be recovered by \(\lambda_t=c_{t+1}\eta_t\) and \(\theta_t=\frac{c_{t+1}}{c_t}\bigl(1-c_t\bigr)\).
        \end{thm}
        
        Because SF specializes SPA to \(\beta_t=1\), this mapping places our nonconvex analysis in direct correspondence with the Heavy-Ball literature. For convex Lipschitz objectives, Heavy-Ball attains the optimal \(\mathcal{O}(1/T)\) rate for the running average of iterates, and --- with a carefully tapered stepsize --- the last iterate enjoys the same \(\mathcal{O}(1/T)\) guarantee \cite{ghadimi2014globalconvergenceheavyballmethod}.
        
    \subsection{Back to Schedule-Free}
        Schedule–Free (SF) unifies PA (momentum) and PR averaging, as defined above. Moreover, for strongly convex problems, it recovers the standard $\mathcal{O}(1/T)$ sub-optimality rate; and general convex problems, the standard $\mathcal{O}\!\bigl(1/\sqrt{T}\bigr)$ rate \cite{defazio2024roadless}. In both cases, SF does this without ever tuning its stepsize to the horizon~$T$. 
        
        The performance of SF in the nonconvex regime is less widely studied, with developments only recently initiated by by Ahn et al. \cite{ahn2024generalframeworkonlinetononconvexconversion}. Assuming \(f\) has $G$-Lipschitz gradients ($\|\nabla f(x)\|\leq G$) and is well-behaved (Equation \ref{eq:well-behaved}), Ahn et al. show that SF achieves the optimal $\mathcal{O}(\lambda^{1/2}\epsilon^{-7/2})$ rate for finding $(\lambda,\epsilon)$-Goldstein stationary points. Moreover, their analysis explains why setting $\beta_t$ close to one and using large base optimizer stepsizes --- choices that empirically worked well in \cite{defazio2024roadless} --- are theoretically justified for nonconvex optimization. Nevertheless, while this provides the first nonconvex guarantees for SF methods, the analysis requires global Lipschitzness of gradients and the well-behaved condition (Equation \ref{eq:well-behaved}), both of which are relatively strong assumptions. Moreover, the parameter choices achieving optimal rates depend on the target accuracy $\epsilon$ \cite{ahn2024generalframeworkonlinetononconvexconversion}.

        \begin{equation}\label{eq:well-behaved}
          f(x)-f(w)=\int_{0}^{1}\bigl\langle\nabla f\bigl(w+t(x-w)\bigr),\,x-w\bigr\rangle\;dt
        \end{equation}
    
        Our analysis tackles the nonconvex landscape of SF while dispensing with the additional assumptions required in \cite{ahn2024generalframeworkonlinetononconvexconversion}. We drop the stringent G-Lipschitz and well-behaved-ness requirements and, in the deterministic setting, assume only that $f$ is $L$ smooth. Furthermore, our analysis provides a simple and flexible groundwork for analyzing SF in nonconvex regimes, in the sense that with a single Lyapunov potential we prove:
        
        \begin{itemize}
        \item an $\mathcal{O}\!\bigl(\frac{1}{\log T}\bigr)$ bound under SF's original constant stepsize \(\eta\) and \(c_{t+1}=1/(t+1)\);
        \item an $\mathcal{O}(\log T/T)$ bias under a linear stepsize $\eta_t=\eta_0(t+1)$ and bounded linear growth condition that is verifiable in PEP experiments;
        \item a continuum of intermediate rates for $c_{t+1}=\bigl(\frac{1}{t+1}\bigr)^\alpha,\,\bigl(\frac{t}{t+1}\bigr)^\alpha$, $\alpha\in[0,1)$, all achieved with horizon-free hyperparameters.
        \end{itemize}

    \subsection{The Performance Estimation Problem (PEP) framework}
        The PEP framework, introduced by Drori and Teboulle \cite{drori2014performance}, is a rigorous framework for analyzing the worst-case convergence rates of gradient-descent style algorithms. To do so, PEP transforms the convergence analysis into a structured semi-definite programming (SDP) formulation which maximizes a worst-case performance metric (e.g., $f(x_T)-f^*$, $||x_T-x_0||^2$, etc.) over a feasible set of functions constrained by their function class (convexity, smoothness properties, etc.).
        
        For example, given the class of smooth convex functions $C^{1,1}_L$, time horizon $T$, some $R>0$ such that $||x_0-x_*|| \le R$, and update rules for $x_t$, the following SDP computes the worst-case sub-optimality ($f(x_T)-f^*$) of the given algorithm over all $f \in C^{1,1}_L$ at time $T$:
        
        \begin{equation}
        \label{eq:pep-setup}
        \begin{alignedat}{2}
            \max_{G\in\mathbb{R}^{(T+1)\times d},\delta_T\mathbb{R}^{T+1}} \quad &LR^2\delta_T\\
            \text{ s.t.} \quad &\text{tr}(G^TA_{i,j}G)\le \delta_i-\delta_j, \quad i < j = 0,\dots,T, \\[3pt]
            \quad &\text{tr}(G^TB_{i,j}G) \le \delta_i-\delta_j, \quad j < i = 0,\dots,T, \\[3pt]
            \quad &\text{tr}(G^TC_iG) \le \delta_i, \quad i=0,\dots,T, \\[3pt]
            \quad &\text{tr}(G^TD_iG+\nu u_i^TG) \le -\delta_i, \quad i=0,\dots,T,
        \end{alignedat}
        \end{equation}
        
        where: \begin{itemize}
            \item $L$ denotes the smoothness number of $f: \mathbb{R}^n \rightarrow\mathbb{R}$
            \item $\delta_i := \frac{1}{L||x_*-x_0||^2}(f(x_i)-f(x_*))$
            \item $G$ is an $(N+1)\times d$ matrix whose $i^{th}$ row is $g_i:=\frac{1}{L||x_*-x_0||}\nabla f(x_i)^T$
            \item $u_i:=e_{i+1}$ (the canonical unit vector)
            \item $\nu$ is any given unit vector in $\mathbb{R}^n$
            \item $A_{i,j}$ and $B_{i,j}$ enforce that gradients and function values behave consistently with $L$-smoothness and convexity
            \item $C_i$ ensures that the function gap $\delta_i$ is nonnegative
            \item $D_{i}$ encodes optimality at $x_*$ ($\nabla f(x_*)=\mathbf 0$ as well as the algorithm's update rules
        \end{itemize}

        Solving this SDP yields a provably tight worst-case performance bound for the given optimization algorithm over the specified function class \cite{taylor2017exact,kim2016exact}. A significant advantage of the PEP framework is its ability to validate theoretical convergence proofs. Indeed, many well-known convergence inequalities, such as those found in Nesterov's classical analyses \cite{nesterov2013introductory}, naturally arise as feasibility conditions within the PEP formulation.
        
        Since its introduction, the PEP framework has been successfully applied to a variety of optimization scenarios, including strongly convex, convex, and more recently, certain nonconvex optimization settings \cite{taylor2019smooth,taylor2021stochastic}. These extensions to nonconvex problems typically involve adapting the objective function to measure squared gradient norms \cite{PEPCourse2024}. Through this automated numerical validation, PEP provides a computationally supported method for validating theoretical worst-case convergence rates in nonconvex scenarios, serving as a powerful complement to our nonconvex analysis of the schedule-free algorithm.

\section{Lyapunov framework}
    We analyze all parameter regimes through a single potential
    \begin{equation}\label{eq:pot-func-def}
        V_t \;=\; f(x_t)-f^\star \;+\; A_t\|\delta_t\|^{2},
        \qquad
        \delta_t:=z_t-x_t,\qquad
        A_t:=\frac{c_{t+1}\bigl(1-L\eta_t c_{t+1}\bigr)}
                    {2\eta_t(1-c_{t+1})^{2}}.
    \end{equation}
    The choice of $A_t$ makes $V_t$ both non–negative and, in the deterministic-setting, monotonically decreasing, thereby tying progress in objective value to shrinkage of the discrepancy
    $\|\delta_t\|^{2}$ between the fast and slow sequences.\\
    
    \begin{lem}\label{lem:pot-descent}
    Let $V_t$ be defined by~\eqref{eq:pot-func-def}.  
    Assume $L\eta_t c_{t+1}\le1$.  Then
    \[
    \begin{aligned}
    V_{t+1}
    \;\le\;&
    V_t
    -\tfrac{c_{t+1}\eta_t}{4}\,\|\nabla f(x_t)\|^{2}\\[4pt]
    &\quad+\Bigl[
       \tfrac{c_{t+1}}{2\eta_t}
      -\tfrac{c_t\bigl(1-L\eta_t c_t\bigr)}{2\eta_t(1-c_t)^{2}}
      -\tfrac{3L^{3}c_{t+1}^{2}\eta_t^{2}}{2}(1-\beta_t)^{2}
      +\tfrac{5L^{2}c_{t+1}\eta_t}{2}(1-\beta_t)^{2}
      \Bigr]\|\delta_t\|^{2}.
    \end{aligned}
    \]
    \end{lem}
    
    The bracketed term is engineered to be \(\le0\) in each hyperparameter schedule we study. Since the total term multiplying \((1-\beta_t)^2\) is nonnegative for \(L\eta_tc_{t+1}\le 1\), our analyzes tighten the one-step descent by taking \(\beta_t=1\). In other words, interpolating SF's updates with PR averaging (\(\beta_t<1\), only appears to add positive noise on the order of \(||\delta_t||^2\).

\subsection{Constant stepsize \(\eta\) and uniform averaging \(c_{t+1}=1/(t+1)\)}
\begin{thm}\label{thm:const-step}
Let Assumption \ref{ass:assumption-1} hold. Consider the iterates defined in \ref{eq:sf-updates} for \(\{c_{t+1},\;\eta_t,\;\beta_t\}=\{\frac{1}{t+1},\;\eta,\;1\}\), where \(\eta\le1/L\).  Then for $t\ge2$,
\[
\min_{2\le t \le T-1}\|\nabla f(x_t)\|^{2}
\;\le\;
\frac{4V_2}{\eta\log T}.
\]
\end{thm}

    In words, Theorem \ref{thm:const-step} shows that the classical schedule-free hyperparameter choice \((c_{t+1}=1/(t+1),\;\beta_t=1)\) drives the squared gradient norm to zero at a suboptimal \(\mathcal{O}\!\bigl(1/\log T\bigr)\) rate, which was also previously shown by Liu et al. (2007) \cite{liu2020improvedanalysisstochasticgradient}. Nevertheless, this bound is horizon-free --- it depends on a single constant step size \(\eta\le 1/L\) and requires no tuning that explicitly depends on \(T\). Nevertheless, Figure \ref{fig:pep-decreasing} suggests that, for \(c_{t+1}=\frac{1}{t+1}\) and constant \(\eta\), \(\min_{2\le t \le T-1}\|\nabla f(x_t)\|^{2}\) may actually be equal to \(\mathcal{O}\!\bigl(1/T\bigr)\), which implies that the rate result for Theorem \ref{thm:const-step} may not be tight. This is described in more detail later in the paper.

\begin{figure}[H]
\centering
\includegraphics[width=\linewidth]{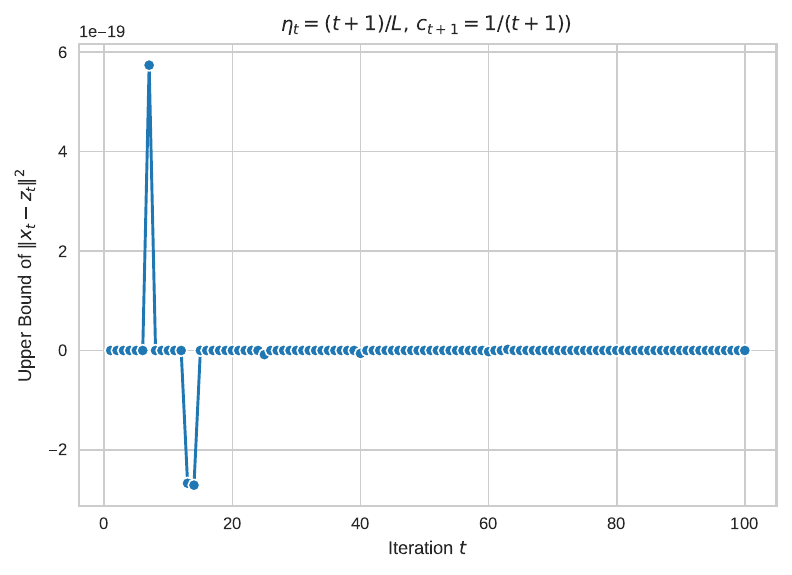}
\caption{Worst-case \(||z_t-x_t||^2\) curve for Assumption \ref{ass:assumption-2}}.
\label{fig:pep-distance}
\end{figure}
\clearpage
\subsection{Linearly growing stepsize \(\eta_t=\eta_0(t+1)\)}
Motivated by PEP numerics we let the stepsize grow, which intuitively
damps $V_t$ faster but risks blowing up $\|\delta_t\|^{2}$.  We assume:\\

\begin{ass}\label{ass:assumption-2}
Consider the iterates defined in \ref{eq:sf-updates} for \(\{c_{t+1},\;\eta_t,\;\beta_t\}=\{\frac{1}{t+1},\;\eta_0(t+1),\;1\}\), where \(\eta_0\le1/L\). Then, there exists $D>0$ such that $\|\delta_t\|^{2}\le D^{2}(t+1)$ for all $t\ge0$.
\end{ass}

From our PEP analysis shown in Figure \ref{fig:pep-distance}, there is strong evidence to support that Assumption \ref{ass:assumption-2} holds. Indeed, we can see that, at least up to \(T=100\) steps, the worst-case \(||z_t-x_t||^2\) quantity can be bounded below a linear function. Thus, we proceed with the following conclusion.

\begin{thm}\label{thm:linear-step}
Let Assumptions \ref{ass:assumption-1} and \ref{ass:assumption-2} hold. Consider the iterates defined in \ref{eq:sf-updates} for \(\{c_{t+1},\;\eta_t,\;\beta_t\}=\{\frac{1}{t+1},\;\eta_0(t+1),\;1\}\), where \(\eta_0\le1/L\).  Then for $t\ge2$,
\[
\min_{2\le t\le T-1}\|\nabla f(x_t)\|^{2}
\;\le\;
\frac{4V_2}{\eta_0 T}
+\mathcal{O}\!\bigl(D^{2}\tfrac{\log T}{T}\bigr).
\]
\end{thm}

The bias term is now \(\mathcal{O}(\log T/T)\approx O(1/T)\) up to a logarithmic correction, nearly matching the deterministic, scheduled momentum bounds previously described; and showing that SF inherits momentum-style acceleration without sacrificing its schedule-free hyperparameters. We also empirically validate this rate in Figure \ref{fig:pep-linear-step}, which shows that for all \(t\) up to \(T=100\) steps, \(\min_{2\le k\le T-1}\|\nabla f(x_t)\|^{2}\times t/\log t\) is bounded above by a constant, thus supporting the rate described in Theorem \ref{thm:linear-step}.

\begin{figure}[H]
\centering
\includegraphics[width=\linewidth]{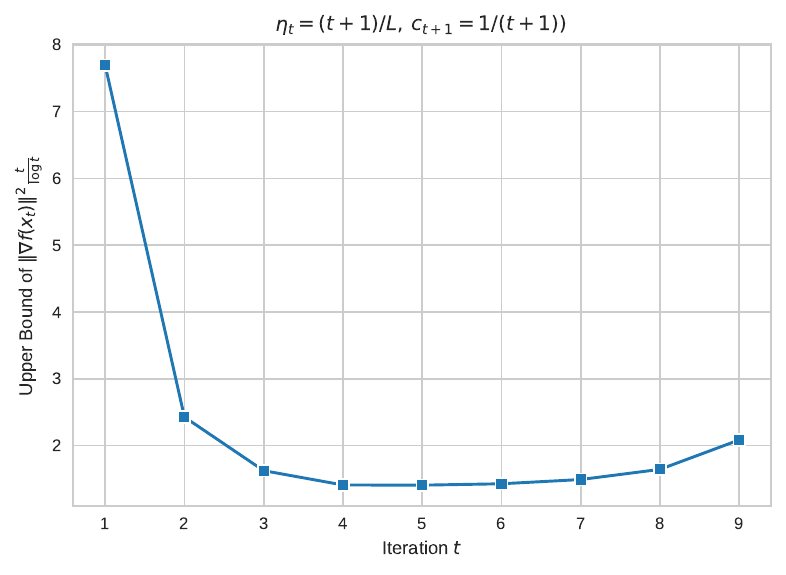}
\caption{Worst-case bias curve for Theorem \ref{thm:linear-step}.}
\label{fig:pep-linear-step}
\end{figure}

\subsection{Constant stepsize with polynomial averaging}
    We generalize the averaging weight to a polynomial decreasing average
    \(c_{t+1}=(t+1)^{-\alpha}\) and polynomial increasing average \(c_{t+1}=t^\alpha/(t+1)^\alpha\) and keep $\beta_t=1$.
    
    \begin{thm}\label{thm:poly-avg}
    Let Assumption \ref{ass:assumption-1} hold. Consider the iterates defined in \ref{eq:sf-updates} for \(\{c_{t+1},\;\eta_t,\;\beta_t\}=\{\bigl(\frac{1}{t+1}\bigr)^\alpha,\;\eta,\;1\}\), where \(\eta\le1/L\).  Then for $t\ge2$,
    \[
    \min_{2\le t \le T-1}\|\nabla f(x_t)\|^{2}
    \;\le\;
    \begin{cases}
    \displaystyle
    \frac{4V_2(1-\alpha)}{\eta\bigl(T^{1-\alpha}-2^{1-\alpha}\bigr)},
    & \alpha\in[0,1),\\[15pt]
    \displaystyle
    \frac{4V_2}{\eta\log T},
    & \alpha=1,\\[15pt]
    \mathcal{O}(1), & \alpha>1.
    \end{cases}
    \]
    Under the same conditions, but for \(c_{t+1}=\bigl(\frac{t}{t+1}\bigr)^\alpha\), and $t\ge2$,
    \[
        \min_{2\le t\le T-1}\|\nabla f(x_t)\|^2\le \frac{4V_2}{\eta T-2\eta -\alpha\eta\log T} + 2\sigma^2.
    \]
    \end{thm}

    Theorem \ref{thm:poly-avg} recovers the result of Theorem \ref{thm:const-step} as a specific case when \(\alpha=1\). In the more general case of \(\alpha\not=1\), we note that the decreasing polynomial averaging (\(c_{t+1}=\bigl(\frac{1}{t+1}\bigr)^\alpha\)) smoothly interpolates between the uniform averaging (\(\alpha=0\)) and the aggressive tail averaging regime \(\alpha=1\) analyzed in Theorem \ref{thm:const-step}. In fact, when \(\alpha\in(0,1)\), the decay \(T^{-(1-\alpha)}\) shows that the more weight in recent iterates accelerates convergence: as \(\alpha\to1^{-}\), the rate formalizes \(O(1/\log T)\). Conversely, placing too much emphasis on recent points (\(\alpha>1\)) causes the bound to collapse to a constant, reflecting that the scheme no longer reduces the gradient norm in the worst case.
    
    Interestingly, a numerical performance estimate (PEP) study of the case \(\alpha=1\) suggests that the true worst‐case decrease under \(c_{t+1}=1/(t+1)\) may actually be \(O(1/T)\) rather than \(O(1/\log T)\), since for all steps \(t\) up to \(T=100\) for \(c_{t+1}=\frac{1}{t+1}\), we observe that \(\min_{2\le k\le t-1}\|\nabla f(x_t)\|^{2}\times t\) is bounded above by a constant.  In other words, the log factor in Theorems \ref{thm:const-step} and \ref{thm:poly-avg} could be an artifact of the proof technique, and a tighter analysis - perhaps via PEP or an improved Lyapunov argument - can recover the optimal rate \(1/T\) for the classic schedule‐free choice.  
    
    Finally, in the case of polynomial increasing average \(c_{t+1}=(t/(t+1))^\alpha\) combines the best of both worlds: it preserves the decay of \(O(1/T)\) (up to a mild logarithm when \(\alpha>0\)). This rate is also empirically validated in Figure \ref{fig:pep-increasing} which shows that, for \(c_{t+1}=(t/(t+1))^\alpha\) and \(\alpha\in[0,1)\), \(\min_{2\le k\le t-1}\|\nabla f(x_t)\|^{2}\times t\) for all \(t\) up to \(T=100\) steps. 
    
\clearpage
\begin{figure}[H]
\centering
\includegraphics[width=\linewidth]{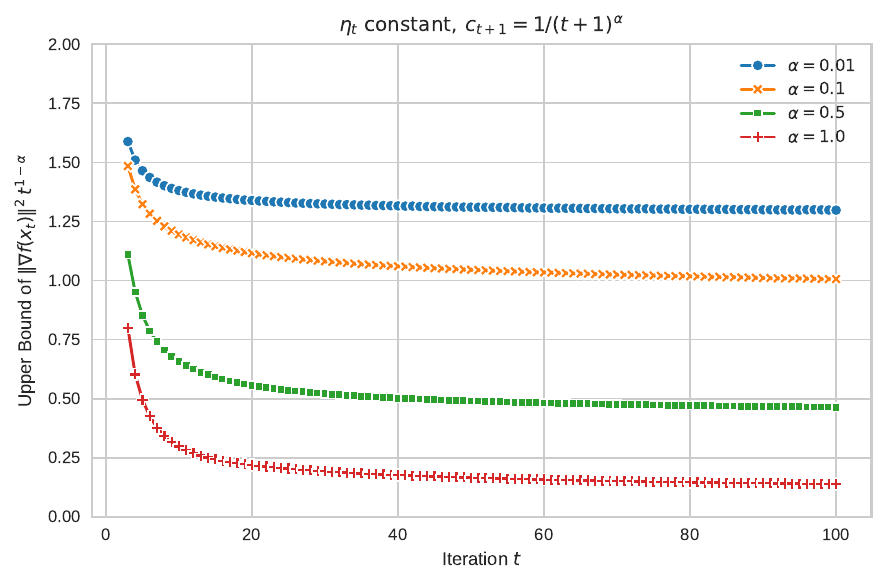}
\caption{Worst-case bias curves for Theorem \ref{thm:poly-avg}, \(c_{t+1}=\bigl(\frac{1}{t+1}\bigr)^\alpha\) for several $\alpha$ values.}
\label{fig:pep-decreasing}
\end{figure}

\begin{figure}[H]
\centering
\includegraphics[width=\linewidth]{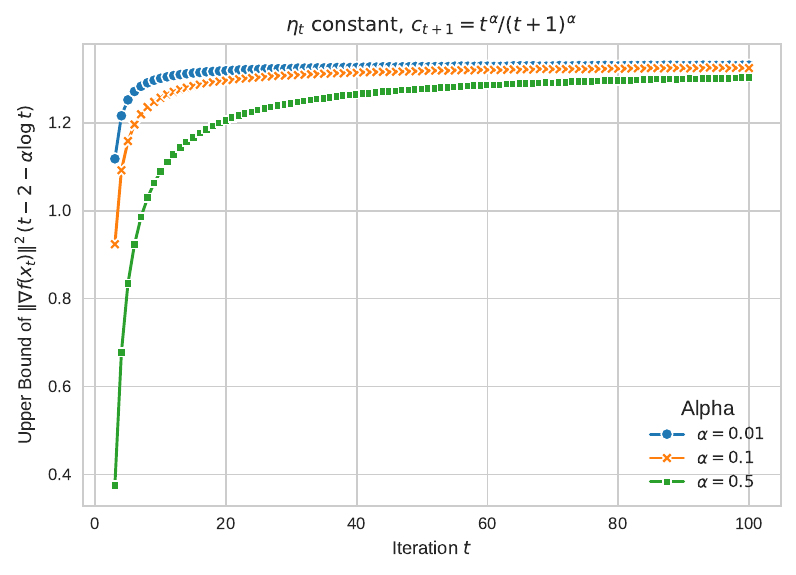}
\caption{Worst-case bias curves for Theorem \ref{thm:poly-avg}, \(c_{t+1}=\bigl(\frac{t}{t+1}\bigr)^\alpha\) for several $\alpha$ values.}
\label{fig:pep-increasing}
\end{figure}
    
\section{Discussion and future work}
    We have developed a unified Lyapunov framework for the Schedule-Free algorithm in the nonconvex smooth setting under only $L$-smoothness and lower-boundedness (Assumption \ref{ass:assumption-1}).  Our main theoretical results show that with the classic SF choice $(c_{t+1}=1/(t+1),\,\eta_t=\eta,\beta_t=1)$ one attains  
    \[
    \min_{2\le t<T}\|\nabla f(x_t)\|^2 = O\!\bigl(1/\log T\bigr),
    \]  
    that by allowing a linearly growing step size $(\eta_t=\eta_0(t+1))$ and a mild bounded-growth condition (Assumption \ref{ass:assumption-2}) one recovers  
    \[
    \min_{2\le t<T}\|\nabla f(x_t)\|^2 = O\!\bigl(\log T/T\bigr),
    \]  
    and that polynomial-averaging weights $c_{t+1}=(t+1)^{-\alpha}$ interpolate between uniform averaging ($\alpha=0$) and the tail-averaging regime ($\alpha=1$) with rates $O\!\bigl(T^{-(1-\alpha)}\bigr)$ for $\alpha\in[0,1)$ (Theorem \ref{thm:poly-avg}).  These rates match or improve upon prior SF analyses in convex and strongly-convex regimes \cite{defazio2024roadless,ahn2024generalframeworkonlinetononconvexconversion} and extend SF into the nonconvex landscape without global Lipschitz or well-behaved assumptions.
    
    To support our theoretical bounds, we used the Performance Estimation Problem (PEP) framework \cite{drori2014performance,taylor2017exact} to compute worst-case curves for $\|\delta_t\|^2$ (Fig.~\ref{fig:pep-distance}), constant-step bias (Fig.~\ref{fig:pep-decreasing}) and linear-step bias (Fig.~\ref{fig:pep-linear-step}).  Across all regimes up to $T=100$, these PEP SDPs support the rates we provided, and in fact suggest possibly tighter bounds that remain to be shown in related SGD+M/SHBM/PA literature. Indeed, the PEP bias curves for $c_{t+1}=1/(t+1)$ appear to decay like $O(1/T)$, hinting that the log-factor is an artifact of our proof technique rather than an inherent limitation of SF. With that being said, our experiments are limited to verifying our analyses for finite time-horizons and do not extrapolate to indefinite \(T\). In fact, the linear-step PEP plot visually exhibits a slight upward tail for larger $t$, and the distance plot under Assumption \ref{ass:assumption-2} shows intermittent spikes even for $t\le100$.  These phenomena suggest that our $O(\log T/T)$ bound in Theorem \ref{thm:linear-step} may not hold uniformly in the true worst case. 
    
    Optimistically, while PEP itself faces these finite-horizon drawbacks, its trends support the conjecture that a tighter theoretical analysis --- perhaps via refined Lyapunov arguments or alternative potential functions --- can recover the optimal $O(1/T)$ rate known for stochastic momentum in nonconvex settings \cite{ghadimi2013stochastic,lei2019variance}. This result could therefore motivate future tighter analysis of nonconvex SF, especially under the conditions assumed in Theorems \ref{thm:const-step} and Theorem \ref{thm:poly-avg}. Establishing an $O(1/T)$ convergence guarantee for the classic schedule-free hyperparameters $(\beta_t=1,c_{t+1}=1/(t+1))$ would align SF with the best known nonconvex momentum rates. Pursuing this sharper bound is our primary theoretical future direction, guided by the empirical PEP evidence.

    Furthermore, although SF requires no explicit stepsize schedule, its optimal momentum or averaging parameter may still depend implicitly on an (approximate) runtime $T$.  In particular, selecting the best fixed $\beta_t$ often relies on knowledge of the total horizon to optimize convergence, shifting the tuning burden from $\eta_t$ to $\beta_t$.  Characterizing SF’s robustness to mis-specified $\beta_t$ and identifying truly universal choices that perform well across a wide range of $T$ remains an open challenge \cite{hägele2024scalinglawscomputeoptimaltraining}. In fact, our current analysis (see Appendix D) restricts to $\beta_t=1$ to eliminate the $(1-\beta_t)^2\|\delta_t\|^2$ penalty in Lemma \ref{lem:pot-descent}.  Although the framework extends to $\beta_t<1$, the appendix derivations show that any $\beta_t<1$ introduces a positive coefficient on $\|\delta_t\|^2$, weakening worst-case descent and contradicting previous empirical results \cite{hägele2024scalinglawscomputeoptimaltraining}. Nonetheless, interpolation with $\beta_t<1$ re-incorporates Polyak–Ruppert averaging, which as mentioned previously, could especially help remedy \(\beta_t\)'s potential liability to mis-specification, as it has been shown to do for stepsize. In sum, another primary future direction for our work is to explore the effects of \(\beta_t\) on nonconvex SF more deeply, analyzing any implicit dependence on training time horizons and any possible associated cancellations of these effects via PR averaging.

\clearpage
\printbibliography
\clearpage

\appendix
\setcounter{thm}{0}
\setcounter{lem}{0}
\section{SGD+M and SPA Equivalence}
\begin{thm}
Define the SGD+M method by the two sequences:
\[
\begin{aligned}
    m_{t+1}&=\theta_tm_t+\nabla f(x_t,\zeta_t),\\
    x_{t+1}&=x_t-\lambda_tm_{t+1},
\end{aligned}
\]
and the SPA sequences as:
\[
\begin{aligned}
    z_{t+1}&=z_t-\eta_t\nabla f(x_t,\zeta_t),\\
    x_{t+1} &= (1-c_{t+1})x_t+c_{t+1}z_{t+1}.
\end{aligned}
\]
Consider the case where \(m_0=0\) for SGD+M and \(z_0=0\) for SPA. Then if \(c_1=\lambda_0/\eta_0\) and for \(t\ge 0\)
\[
    \eta_{t+1}=\frac{\eta_t-\lambda_t}{\theta_{t+1}},\quad c_{t+1}=\frac{\lambda_t}{\eta_t},
\]
the \(x\) sequence produced by the SPA method is identical to the \(x\) sequence produced by the SGD+M method.
\end{thm}

\begin{proof}
Consider the base case where \(x_0=z_0\). Then for SGD+M:
\begin{equation}
\label{eq:sgdm-0}
\begin{aligned}
    m_1 = &\nabla f(x_0, \zeta_t)\\
    \therefore x_1 = x_0 &-\lambda_0\nabla f(x_0,\zeta_t)
\end{aligned}
\end{equation}

and for the SPA form:
\begin{equation}
\label{eq:spa-0}
\begin{aligned}
    z_1 = x_0&-\eta_0\nabla f(x_0, \zeta_0)\\
    x_1 = (1-c_0)x_0&+c_0(x_0-\eta_0\nabla f(x_0, \zeta_0))\\
    =x_0-c_0&\eta_0\nabla f(x_0, \zeta_0)
\end{aligned}
\end{equation}

Clearly, Equation \ref{eq:sgdm-0} is equivalent to Equation \ref{eq:spa-0} when \(\lambda_0=c_0\eta_0\).

Now consider \(t>0\). We will define \(z_t\) in terms of quantities in the SGD+M method, then show that with this definition the step-to-step changes in \(z\) correspond exactly to the SPA method. In particular, let:
\begin{equation}
\label{eq:zt-mt}
    z_t = x_t - \biggl(\frac{1}{c_t}-1\biggr)\lambda_{t-1}m_t.
\end{equation}

Then
\begin{equation}
\begin{aligned}
    z_{t+1}&=x_{t+1}-\biggl(\frac{1}{c_{t+1}}-1\biggr)\lambda_{t}m_{t+1}\\
    &=x_t-\lambda_tm_{t+1}- \biggl(\frac{1}{c_{t+1}}-1\biggr)\lambda_{t}m_{t+1}\\
    &=z_t + \biggl(\frac{1}{c_t}-1\biggr)\lambda_{t-1}m_t - \frac{\lambda_t}{c_{t+1}}(\theta_tm_t+\nabla f(x_t, \zeta_t))\\
    &=z_t + \biggl[\biggl(\frac{1}{c_t}-1\biggr)\lambda_{t-1}-\frac{\lambda_t}{c_{t+1}}\theta_t\biggr]m_t - \frac{\lambda_t}{c_{t+1}}\nabla f(x_t, \zeta_t).
\end{aligned}
\end{equation}

This is equivalent to the SPA step
\[
    z_{t+1}=z_t-\eta_t\nabla f(x_t, \zeta_t),
\]

as long as \(\frac{\lambda_t}{c_{t+1}}=\eta_t\) and
\[
\begin{aligned}
    0&=\biggl(\frac{1}{c_t}-1\biggr)\lambda_{t-1}-\frac{\lambda_t}{c_{t+1}}\theta_t\\
    &=(\eta_{t-1}-\lambda_{t-1})-\eta_t\theta_t,\\
    &\quad\text{i.e., } \eta_t=\frac{\eta_{t-1}-\lambda_{t-1}}{\theta_t}.
\end{aligned}
\]

Using this definition of the \(z\) sequence, we can rewrite the SGD+M \(x\) sequence using a rearrangement of Equation \ref{eq:zt-mt}:
\[
\begin{aligned}
    m_{t+1}&=\biggl(\frac{1}{c_{t+1}}-1\biggr)^{-1}\lambda_t^{-1}(x_{t+1}-z_{t+1})\\
    &= \frac{c_{t+1}}{1-c_{t+1}}\lambda_t^{-1}(x_{t+1}-z_{t+1}),
\end{aligned}
\]
as
\[
\begin{aligned}
    x_{t+1}&= x_t-\lambda_tm_{t+1}\\
    &= x_t - \frac{c_{t+1}}{1-c_{t+1}}(x_{t+1}-z_{t+1})\\
    &=x_t - \frac{c_{t+1}}{1-c_{t+1}}x_{t+1}+ \frac{c_{t+1}}{1-c_{t+1}}z_{t+1}\\
    &= (1-c_{t+1})x_{t+1}+c_{t+1}z_{t+1},
\end{aligned}
\]
matching the SPA update.
\end{proof}
\clearpage
\section{SHBM and SPA Equivalence}
\begin{thm}
Define the SHBM method by the sequence:
\[
\begin{aligned}
    x_{t+1}=x_t-\lambda_t\nabla f(x_t, \zeta_t)+\theta_t(x_t-x_{t-1}),
\end{aligned}
\]
and the SPA sequences as:
\[
\begin{aligned}
    z_{t+1}&=z_t-\eta_t\nabla f(x_t,\zeta_t),\\
    x_{t+1} &= (1-c_{t+1})x_t+c_{t+1}z_{t+1}.
\end{aligned}
\]
Consider the case where \(x_0=0\) for SHBM and \(z_0=0\) for SPA. Then if for all \(t\ge 0\)
\[
    \lambda_t = c_{t+1}\eta_t,\quad \theta_t = \frac{c_{t+1}(1-c_t)}{c_t},
\]
the \(x\) sequence produced by the SPA method is identical to the \(x\) sequence produced by the SHBM method.
\end{thm}

\begin{proof}
Consider the base case where $t=0$. For SHBM with $x_{0} = 0$:
\begin{equation}
\label{eq:SHBM-0}
\begin{aligned}
    x_1 &= x_0 - \lambda_0 \nabla f(x_0, \zeta_0) + \theta_0 (x_0 - x_{-1}) \\
        &= x_0 - \lambda_0 \nabla f(x_0, \zeta_0)
\end{aligned}
\end{equation}

For the SPA form with $z_0 = x_0$:
\begin{equation}
\label{eq:spa-0}
\begin{aligned}
    z_1 &= z_0 - \eta_0 \nabla f(x_0, \zeta_0) \\
        &= x_0 - \eta_0 \nabla f(x_0, \zeta_0) \\
    x_1 &= (1 - c_1) x_0 + c_1 z_1 \\
        &= x_0 - c_1 \eta_0 \nabla f(x_0, \zeta_0)
\end{aligned}
\end{equation}

Equations Equation \ref{eq:SHBM-0} and Equation \ref{eq:spa-0} are identical when $\lambda_0 = c_1 \eta_0$.

Now consider $t > 0$. We will define $z_t$ in terms of quantities in the SHBM method, then show that with this definition the step-to-step changes in $z$ correspond exactly to the SPA method. Let:
\begin{equation}
\label{eq:zt-xt}
    z_t = x_t + \frac{1 - c_t}{c_t} (x_t - x_{t-1}).
\end{equation}

Then:
\begin{equation}
\begin{aligned}
    z_{t+1} &= x_{t+1} + \frac{1 - c_{t+1}}{c_{t+1}} (x_{t+1} - x_t) \\
            &= \left(1 + \frac{1 - c_{t+1}}{c_{t+1}}\right) x_{t+1} - \frac{1 - c_{t+1}}{c_{t+1}} x_t \\
            &= \frac{1}{c_{t+1}} x_{t+1} - \frac{1 - c_{t+1}}{c_{t+1}} x_t
\end{aligned}
\end{equation}

Substituting the SHBM update for $x_{t+1}$:
\begin{equation}
\begin{aligned}
    z_{t+1} &= \frac{1}{c_{t+1}} \left[ x_t - \lambda_t \nabla f(x_t, \zeta_t) + \theta_t (x_t - x_{t-1}) \right] - \frac{1 - c_{t+1}}{c_{t+1}} x_t \\
            &= z_t + \frac{1 - c_t}{c_t} (x_t - x_{t-1}) - \frac{\lambda_t}{c_{t+1}} \nabla f(x_t, \zeta_t) \\
            &\quad + \frac{\theta_t}{c_{t+1}} (x_t - x_{t-1}) - \frac{1 - c_{t+1}}{c_{t+1}} (x_t - z_t)
\end{aligned}
\end{equation}

This simplifies to the SPA update $z_{t+1} = z_t - \eta_t \nabla f(x_t, \zeta_t)$ when:
\[
\frac{\lambda_t}{c_{t+1}} = \eta_t \quad \text{and} \quad \theta_t = \frac{c_{t+1}(1 - c_t)}{c_t}.
\]

Finally, the SHBM $x$ update can be rewritten using Equation \ref{eq:zt-xt}:
\[
\begin{aligned}
    x_{t+1} &= (1 - c_{t+1}) x_t + c_{t+1} z_{t+1} \\
            &= (1 - c_{t+1}) x_t + c_{t+1} \left( z_t - \eta_t \nabla f(x_t, \zeta_t) \right),
\end{aligned}
\]
which matches the SPA update by construction.
\end{proof}

\clearpage
\section{Potential Function Descent}
\subsection{Stochastic assumptions}
The following assumption is effective throughout the rest of our proofs and is standard in stochastic nonconvex optimization, but can be ignored in deterministic settings.\\
\begin{ass}
\label{ass:assumption-3}\hfill
\begin{enumerate} 
    \item \textbf{Unbiasedness:} At each iteration \(t\), \(\nabla f(x_t,\zeta_t)\) satisfies \(\mathbb{E}_{\zeta_t}\bigl[\nabla f(x_t,\zeta_t)\bigr]=\nabla f(x_t)\).
    \item \textbf{Independent samples:} The random samples \(\{\zeta_t\}_{t=0}^{\infty}\) are independent.
    \item \textbf{Bounded variance:} The variance of \(\nabla f(x_t,\zeta_t)\) with respect to \(\zeta_t\) satisfies \(\textup{Var}_{\zeta_t}(\nabla f(x_t, \zeta_t) = \mathbb{E}_{\zeta_t}\bigl[||\nabla f(x_t, \zeta_t)-\nabla f(x_t)||^2\bigr]\le\sigma^2\) for some \(\sigma^2\ge0\).
\end{enumerate}
\end{ass}

\subsection{Claim 1}
\begin{claim}
Define Schedule-Free by the three sequences:
\begin{equation}
\label{eq:sf-updates-appendix}
\begin{aligned}
        y_t &= (1 - \beta_t) z_t + \beta_t x_t,\\
        z_{t+1} &= z_t - \eta_t\nabla f(y_t,\zeta_t)\\
        x_{t+1}&=(1-c_{t+1}) x_t+c_{t+1} z_{t+1},
\end{aligned}
\end{equation}
and let Assumptions \ref{ass:assumption-1} and -\ref{ass:assumption-3} hold. Define \(\delta_{t+1}:=z_{t+1}-x_{t+1}\). Then,
\[
    \delta_{t+1} = (1-c_{t+1})(\delta_t-\eta_t\nabla f(y_t,\zeta_t)),
\]
which implies that
\[
\begin{aligned}
    &\quad\quad||\delta_{t+1}||^2 = (1-c_{t+1})^2\bigl(||\delta_t||^2-2\eta_t\nabla f(y_t, \zeta_t)^\top \delta_t +\eta_t^2||\nabla f(y_t, \zeta_t)||^2\bigr)\\
    &\Rightarrow \mathbb{E}\bigl[||\delta_{t+1}||^2\bigr] \le (1-c_{t+1})^2\bigl(||\delta_t||^2-2\eta_t\mathbb{E}\bigl[\nabla f(y_t)^\top \delta_t\bigr] +\eta_t^2||\nabla f(y_t)||^2\bigr) + (1-c_{t+1})^2\eta_t^2\sigma^2.
\end{aligned}
\]
\end{claim}

\begin{proof}
First, we establish the recurrence for \(\delta_{t+1}\). From the definition of \(x_{t+1}\) in Equation \ref{eq:sf-updates}:
\[
x_{t+1} = (1 - c_{t+1}) x_t + c_{t+1} z_{t+1},
\]
we can express the difference:
\[
\begin{aligned}
\delta_{t+1} &= z_{t+1} - x_{t+1} \\
&= z_{t+1} - (1 - c_{t+1}) x_t - c_{t+1} z_{t+1} \\
&= (1 - c_{t+1}) (z_{t+1} - x_t).
\end{aligned}
\]

Now, using the update for \(z_{t+1}\) from Equation \ref{eq:sf-updates}:
\[
z_{t+1} = z_t - \eta_t \nabla f(y_t, \zeta_t),
\]
we substitute to get:
\[
\begin{aligned}
\delta_{t+1} &= (1 - c_{t+1}) (z_t - \eta_t \nabla f(y_t, \zeta_t) - x_t) \\
&= (1 - c_{t+1}) (\delta_t - \eta_t \nabla f(y_t, \zeta_t)),
\end{aligned}
\]
where we used \(\delta_t = z_t - x_t\). This proves the first claim.

For the norm bound, we square both sides:
\[
\begin{aligned}
\|\delta_{t+1}\|^2 &= (1 - c_{t+1})^2 \|\delta_t - \eta_t \nabla f(y_t, \zeta_t)\|^2 \\
&= (1 - c_{t+1})^2 \bigl(\|\delta_t\|^2 - 2\eta_t \nabla f(y_t, \zeta_t)^\top \delta_t + \eta_t^2 \|\nabla f(y_t, \zeta_t)\|^2\bigr).
\end{aligned}
\]

Taking expectations and using that \(\mathbb{E}[\nabla f(y_t, \zeta_t)] = \nabla f(y_t)\) and \(\mathbb{E}[\|\nabla f(y_t, \zeta_t) - \nabla f(y_t)\|^2] \le \sigma^2\):
\[
\begin{aligned}
\mathbb{E}\bigl[\|\delta_{t+1}\|^2\bigr] &= (1 - c_{t+1})^2 \bigl(\|\delta_t\|^2 - 2\eta_t \mathbb{E}\bigl[\nabla f(y_t)^\top \delta_t\bigr] + \eta_t^2 \mathbb{E}\bigl[\|\nabla f(y_t, \zeta_t)\|^2\bigr]\bigr) \\
&\le (1 - c_{t+1})^2 \bigl(\|\delta_t\|^2 - 2\eta_t \mathbb{E}\bigl[\nabla f(y_t)^\top \delta_t\bigr] + \eta_t^2 \|\nabla f(y_t)\|^2 + \eta_t^2 \sigma^2\bigr),
\end{aligned}
\]
where the inequality follows from expanding \(\mathbb{E}[\|\nabla f(y_t, \zeta_t)\|^2] = \|\nabla f(y_t)\|^2 + \mathbb{E}[\|\nabla f(y_t, \zeta_t) - \nabla f(y_t)\|^2]\). This completes the proof.
\end{proof}

\subsection{Claim 2}
\begin{claim}
Define Schedule-Free and \(\delta_t\) as in Claim 1 and further let Assumptions \ref{ass:assumption-1} and \ref{ass:assumption-3} hold. Then,
\[
    \mathbb{E}\bigl[||\nabla f(y_t,\zeta_t)||^2\bigr]\ge \frac{1}{2}||\nabla f(x_t)||^2-2L^2(1-\beta_t)^2||\delta_t||^2 - \sigma^2.
\]
\end{claim}

\begin{proof}
We begin by relating the gradient at \(y_t\) to the gradient at \(x_t\). Using the \(L\)-smoothness of \(f\):
\begin{equation}
\label{eq:l-smoothness}
\|\nabla f(y_t) - \nabla f(x_t)\| \leq L\|y_t - x_t\|.
\end{equation}

From the definition of \(y_t\) and \(\delta_t\):
\begin{equation}
\label{eq:ytxt-ztxt}
y_t - x_t = (1-\beta_t)(z_t - x_t) = (1-\beta_t)\delta_t.
\end{equation}

Substituting Equation \ref{eq:ytxt-ztxt} into Equation \ref{eq:l-smoothness}:
\[
    \|\nabla f(y_t) - \nabla f(x_t)\| \leq L(1-\beta_t)||\delta_t||.
\]

Thus, by the triangle inequality:
\[
\begin{aligned}
&||\nabla f(x_t)||-||\nabla f(y_t)||\le ||\nabla f(y_t) - \nabla f(x_t)|| \le L(1-\beta_t)||\delta_t||\\
\Rightarrow\
&\|\nabla f(y_t)\| \geq \|\nabla f(x_t)\| - L(1-\beta_t)\|\delta_t\|.
\end{aligned}
\]

Squaring both sides and using \((a-b)^2 \geq a^2 - 2ab\) for \(a,b \geq 0\):
\[
\|\nabla f(y_t)\|^2 \geq \|\nabla f(x_t)\|^2 - 2L(1-\beta_t)\|\nabla f(x_t)\|\|\delta_t\|.
\]

Taking expectations and using that \(\mathbb{E}[\|\nabla f(y_t,\zeta_t)\|^2] \ge \|\nabla f(y_t)\|^2 - \sigma^2\):
\[
\mathbb{E}\bigl[\|\nabla f(y_t,\zeta_t)\|^2\bigr] \geq \|\nabla f(x_t)\|^2 - 2L(1-\beta_t)\|\nabla f(x_t)\|\|\delta_t\| - \sigma^2.
\]

Finally, using Young's inequality \(ab \leq \frac{a^2}{2c} + \frac{cb^2}{2}\) for \(c>0\):
\[
\mathbb{E}\bigl[\|\nabla f(y_t,\zeta_t)\|^2\bigr] \geq \frac{1}{2}\|\nabla f(x_t)\|^2 -2L^2(1-\beta_t)^2||\delta_t||^2 - \sigma^2.
\]
where we used \(c=2\) to yield the desired result.
\end{proof}

\subsection{Claim 3}
\begin{claim}
Define Schedule-Free and \(\delta_t\) as in Claim 1 and let Assumption \ref{ass:assumption-1} hold. Then,
\[
    ||\nabla f(y_t)||^2\le ||\nabla f(x_t)||^2+L^2(1-\beta_t)^2||\delta_t||^2.
\]
\end{claim}

\begin{proof}
We begin by relating the gradient at \(y_t\) to the gradient at \(x_t\). Using the \(L\)-smoothness of \(f\):
\begin{equation}
\label{eq:l-smoothness}
\|\nabla f(y_t) - \nabla f(x_t)\| \leq L\|y_t - x_t\|.
\end{equation}

From the definition of \(y_t\) and \(\delta_t\):
\begin{equation}
\label{eq:ytxt-ztxt}
y_t - x_t = (1-\beta_t)(z_t - x_t) = (1-\beta_t)\delta_t.
\end{equation}

Substituting Equation \ref{eq:ytxt-ztxt} into Equation \ref{eq:l-smoothness}:
\[
    \|\nabla f(y_t) - \nabla f(x_t)\| \leq L(1-\beta_t)||\delta_t||.
\]
Therefore, by the Triangle Inequality:
\[
\begin{aligned}
    ||\nabla f(y_t)||^2 &= ||\nabla f(x_t) - \bigl(\nabla f(x_t)-\nabla f(y_t)\bigr)||^2\\
    &\le ||\nabla f(x_t)||^2 + ||\nabla f(x_t)-\nabla f(y_t)||^2\\
    &\le ||\nabla f(x_t)||^2 + L^2(1-\beta_t)^2||\delta_t||^2.
\end{aligned}
\]
This yields the desired result.
\end{proof}

\subsection{Lemma 1}
\begin{lem}
Define Schedule-Free by the three sequences:
\begin{equation}
\label{eq:sf-updates-appendix-2}
\begin{aligned}
        y_t &= (1 - \beta_t) z_t + \beta_t x_t,\\
        z_{t+1} &= z_t - \eta_t\nabla f(y_t,\zeta_t)\\
        x_{t+1}&=(1-c_{t+1}) x_t+c_{t+1} z_{t+1}.
\end{aligned}
\end{equation}

Let \(V_t\) be defined as
\begin{equation}
\label{eq:potential-function-definition}
    V_t = f(x_t)+A_t||\delta_t||^2,\quad A_t=\frac{c_{t+1}(1-L\eta_t c_{t+1})}{2\eta_t(1-c_{t+1})^2}.
\end{equation}

Then, if Assumptions \ref{ass:assumption-1} and \ref{ass:assumption-3} hold, 
\[
\begin{aligned}
    \mathbb{E}\bigl[V_{t+1}\bigr] &\le V_t
-\frac{c_{t+1}\eta_t}{4}
  \|\nabla f(x_t)\|^2 + \frac{c_{t+1}\eta_t}{2}\sigma^2\\
&\quad+\Bigr(\frac{c_{t+1}}{2\eta_t}-\frac{c_t(1-L\eta_t c_t)}{2\eta_t(1-c_t)^2}
-\frac{3L^3c_{t+1}^2\eta_t^2}{2}(1-\beta_t)^2
+\frac{5L^2c_{t+1}\eta_t}{2}(1-\beta_t)^2\Bigr)||\delta_t||^2\\
\end{aligned}
\]
\end{lem}

\begin{proof}
We begin by considering the expectation of \(V_{t+1}\) conditioned on \(\zeta_t\):
\begin{equation}
\label{eq:expected-next-pot}
    \mathbb{E}\bigl[V_{t+1}\bigr] = \mathbb{E}\bigl[f(x_{t+1})\bigr]+A_{t+1}\mathbb{E}\bigl[||\delta||^2\bigr]
\end{equation}

By the \(L\)-smoothness of \(f\), for any \(u,v\in\mathbb{R}^{n}\) we have
\begin{equation}
\label{eq:l-smoothness-lem3}
\begin{aligned}
f(u)&\leq f(v)+\nabla f(v)^{\top}(u-v)+\frac{L}{2}\|u-v\|^{2}\\
\Rightarrow\mathbb{E}\bigl[f(u)\bigr] &\le \mathbb{E}\bigl[f(v)\bigr]+\mathbb{E}\bigl[\nabla f(v)^{\top}(u-v)\bigr]+\frac{L}{2}\mathbb{E}\bigl[\|u-v\|^{2}\bigr]
\end{aligned}
\end{equation}

Substituting \(u=x_{t+1}\), \(v=x_{t}\) into Equation \ref{eq:l-smoothness-lem3} and noting that 
\[
x_{t+1}=x_{t}+c_{t+1}(\delta_{t}-\eta_t\nabla f(y_{t})),
\]
we obtain
\[
\begin{aligned}
\mathbb{E}\bigl[f(x_{t+1})\bigr] &\leq \mathbb{E}\bigl[f(x_{t})\bigr]+c_{t+1}\mathbb{E}\bigl[\nabla f(y_{t})^{\top}(\delta_{t}-\eta\nabla f(y_{t}))\bigr]+\frac{Lc_{t+1}^{2}}{2}\mathbb{E}\bigl[\|\delta_{t}-\eta_t\nabla f(y_{t})\|^{2}\bigr]\\
&= f(x_{t})-c_{t+1}\eta_t\mathbb{E}\bigl[\|\nabla f(y_{t})\|^{2}\bigl]+c_{t+1}\mathbb{E}\bigl[\nabla f(y_{t})^{\top}\delta_{t}\bigr]\\
&\quad +\frac{Lc_{t+1}^{2}}{2}\Big{(}\mathbb{E}\bigl[\|\delta_{t}\|^{2}\bigr]-2\eta_t\mathbb{E}\bigl[\nabla f(y_{t})^{\top}\delta_{t}\bigl]+\eta_t^{2}\mathbb{E}\bigl[\|\nabla f(y_{t})\|^{2}\bigl]\Big{)}.
\end{aligned}
\]

Rearranging, we have
\begin{equation}
\label{eq:descent-on-f}
\begin{aligned}
\mathbb{E}\bigr[f(x_{t+1})\bigr]\leq f(x_{t}) - \left(c_{t+1}\eta_t-\frac{Lc_{t+1}^{2}\eta_t^{2}}{2}\right)\mathbb{E}\bigl[\|\nabla f(y_{t})\|^{2}\bigr]\\
+ (c_{t+1}-Lc_{t+1}^{2}\eta_t)\mathbb{E}\bigl[\nabla f(y_{t})^{\top}\delta_{t}\bigr] + \frac{Lc_{t+1}^{2}}{2}\|\delta_{t}\|^{2}.
\end{aligned}
\end{equation}

Next, from Claim 1 we have:
\begin{equation}
\label{eq:delta-recurse-restate}
\mathbb{E}\bigl[||\delta_{t+1}||^2\bigr] \le (1-c_{t+1})^2\bigl(||\delta_t||^2-2\eta_t\mathbb{E}\bigl[\nabla f(y_t)^\top \delta_t\bigr] +\eta_t^2||\nabla f(y_t)||^2\bigr) + (1-c_{t+1})^2\eta_t^2\sigma^2
\end{equation}

Thus, after substituting Equation \ref{eq:descent-on-f} and Equation \ref{eq:delta-recurse-restate} into Equation \ref{eq:expected-next-pot}, we have:
\begin{equation}
\label{eq:initial-desc-any-At}
\begin{aligned}
\mathbb{E}\bigl[V_{t+1}\bigr]&\le f(x_{t})-c_{t+1}\eta_t\mathbb{E}\bigl[\|\nabla f(y_{t})\|^{2}\bigl]+c_{t+1}\mathbb{E}\bigl[\nabla f(y_{t})^{\top}\delta_{t}\bigr]\\
&\quad\quad +\frac{Lc_{t+1}^{2}}{2}\Big{(}\|\delta_{t}\|^{2}-2\eta_t\mathbb{E}\bigl[\nabla f(y_{t})^{\top}\delta_{t}\bigl]+\eta_t^{2}\mathbb{E}\bigl[\|\nabla f(y_{t})\|^{2}\bigl]\Big{)}\\
&\quad\quad+(1-c_{t+1})^2A_{t+1}\bigl(||\delta_t||^2-2\eta_t\mathbb{E}\bigl[\nabla f(y_t)^\top \delta_t\bigr] +\eta_t^2||\nabla f(y_t)||^2\bigr) + (1-c_{t+1})^2\eta_t^2A_{t+1}\sigma^2\\
&=V_t+\Bigr(c_{t+1}(1-L\eta_tc_{t+1})-2\eta_tA_{t+1}(1-c_{t+1})^2\Bigr)\mathbb{E}\bigl[\nabla f(y_t)^\top \delta_t\bigr]\\
&\quad\quad+\Bigl(\frac{Lc_{t+1}^2\eta_t^2}{2}-c_{t+1}\eta_t\Bigr)\mathbb{E}\bigl[||\nabla f(y_t)||^2\bigr]+(1-c_{t+1})^2\eta_t^2A_{t+1}||\nabla f(y_t)||^2\\
&\quad\quad + \Bigr(A_{t+1}(1-c_{t+1})^2-A_t+\frac{Lc_{t+1}^2}{2}\bigl)||\delta_t||^2 + (1-c_{t+1})^2\eta_t^2A_{t+1}\sigma^2,
\end{aligned}
\end{equation}

where we substitute \(V_t=f(x_t)+A_t||\delta_t||^2\). Let us require that \(1-Lc_{t+1}\eta_t\geq 0\) and set
\begin{equation}
\label{eq:at1-def}
A_{t+1}=\frac{c_{t+1}(1-Lc_{t+1}\eta_t)}{2\eta_t(1-c_{t+1})^{2}}.
\end{equation}

Substituting into Equation \ref{eq:initial-desc-any-At}, this choice makes the inner-product term zero. Then we get
\begin{equation}
\label{eq:second-desc-any-At}
\begin{aligned}
\mathbb{E}\bigl[V_{t+1}\bigr]&\le V_t+\Bigl(\frac{Lc_{t+1}^2\eta_t^2}{2}-c_{t+1}\eta_t\Bigr)\mathbb{E}\bigl[||\nabla f(y_t)||^2\bigr]\\
&\quad\quad+(1-c_{t+1})^2\eta_t^2A_{t+1}||\nabla f(y_t)||^2+(1-c_{t+1})^2\eta_t^2A_{t+1}\sigma^2\\
&\quad\quad+\Bigr(A_{t+1}(1-c_{t+1})^2-A_t+\frac{Lc_{t+1}^2}{2}\bigl)||\delta_t||^2. \\
\end{aligned}
\end{equation}

Observe that for \(1-Lc_{t+1}\eta_t\geq 0\) we have \(\frac{Lc_{t+1}^2\eta_t^2}{2}-c_{t+1}\eta_t\le0\). Therefore, we can substitute for \(\mathbb{E}\bigl[||\nabla f(y_t)||^2\bigr]\) (Claim 2) and \(||\nabla f(y_t)||^2\) (Claim 3) into Equation \ref{eq:second-desc-any-At} as follows:
\begin{equation}
\label{eq:third-desc-any-At}
\begin{aligned}
\mathbb{E}\bigl[V_{t+1}\bigr]&\le V_t+\Bigl(\frac{Lc_{t+1}^2\eta_t^2}{2}-c_{t+1}\eta_t\Bigr)\biggl[\frac{1}{2}||\nabla f(x_t)||^2-2L^2(1-\beta_t)^2||\delta_t||^2 - \sigma^2\biggr]\\
&\quad\quad+(1-c_{t+1})^2\eta_t^2A_{t+1}\Bigl[||\nabla f(x_t)||^2+L^2(1-\beta_t)^2||\delta_t||^2\Bigr]+(1-c_{t+1})^2\eta_t^2A_{t+1}\sigma^2\\
&\quad\quad+\Bigr(A_{t+1}(1-c_{t+1})^2-A_t+\frac{Lc_{t+1}^2}{2}\bigl)||\delta_t||^2 \\
&= V_t
+ \Bigl(\tfrac{Lc_{t+1}^2\eta_t^2}{4}
  - \tfrac{c_{t+1}\eta_t}{2}
  + (1-c_{t+1})^2\eta_t^2A_{t+1}\Bigr)\|\nabla f(x_t)\|^2\\
&\quad
+ \Bigl[L^2(1-c_{t+1})^2(1-\beta_t)^2\eta_t^2A_{t+1}-2L^2(1-\beta_t)^2\bigl(\tfrac{Lc_{t+1}^2\eta_t^2}{2}
  - c_{t+1}\eta_t\bigr)\\
  &\quad\quad+A_{t+1}(1-c_{t+1})^2- A_t+ \tfrac{Lc_{t+1}^2}{2}\Bigr]\|\delta_t\|^2\\
&\quad
+ \Bigl(-\tfrac{Lc_{t+1}^2\eta_t^2}{2}
  + c_{t+1}\eta_t
  + (1-c_{t+1})^2\eta_t^2A_{t+1}\Bigr)\sigma^2\\
&\le V_t
+ \frac{1}{2}\Bigl(\tfrac{Lc_{t+1}^2\eta_t^2}{2}
  - c_{t+1}\eta_t
  + (1-c_{t+1})^2\eta_t^2A_{t+1}\Bigr)\|\nabla f(x_t)\|^2\\
&\quad
+ \Bigl[L^2(1-c_{t+1})^2(1-\beta_t)^2\eta_t^2A_{t+1}-2L^2(1-\beta_t)^2\bigl(\tfrac{Lc_{t+1}^2\eta_t^2}{2}
  - c_{t+1}\eta_t\bigr)\\
  &\quad\quad+A_{t+1}(1-c_{t+1})^2- A_t+ \tfrac{Lc_{t+1}^2}{2}\Bigr]\|\delta_t\|^2\\
&\quad
+ \Bigl(-\tfrac{Lc_{t+1}^2\eta_t^2}{2}
  + c_{t+1}\eta_t
  + (1-c_{t+1})^2\eta_t^2A_{t+1}\Bigr)\sigma^2\\
\end{aligned}
\end{equation}

Substituting our definition of \(A_{t+1}\) from Equation \ref{eq:at1-def} into Equation \ref{eq:third-desc-any-At}, we get
\[
\begin{aligned}
\mathbb{E}\bigl[V_{t+1}\bigr] &\le V_t
-\frac{c_{t+1}\eta_t}{4}
  \|\nabla f(x_t)\|^2 + \frac{c_{t+1}\eta_t}{2}\sigma^2\\
&\quad+\Bigr(\frac{c_{t+1}}{2\eta_t}-\frac{c_t(1-L\eta_t c_t)}{2\eta_t(1-c_t)^2}
-\frac{3L^3c_{t+1}^2\eta_t^2}{2}(1-\beta_t)^2
+\frac{5L^2c_{t+1}\eta_t}{2}(1-\beta_t)^2\Bigr)||\delta_t||^2\\
\end{aligned}
\]
which is the desired result.
\end{proof}

\clearpage
\section{Main Theorems}

\subsection{Theorem 3}
\begin{thm}
Define Schedule-Free by the three sequences:
\[
\begin{aligned}
        y_t &= (1 - \beta_t) z_t + \beta_t x_t,\\
        z_{t+1} &= z_t - \eta_t\nabla f(y_t,\zeta_t)\\
        x_{t+1}&=(1-c_{t+1}) x_t+c_{t+1} z_{t+1},
\end{aligned}
\]
and let \(V_t\) be defined as
\[
    V_t = f(x_t)+A_t||\delta_t||^2,\quad A_t=\frac{c_{t+1}(1-L\eta_t c_{t+1})}{2\eta_t(1-c_{t+1})^2},
\]
where Assumptions \ref{ass:assumption-1} and \ref{ass:assumption-3} hold. Then for constant \(\eta \le 1/L\), \(c_{t+1} = 1/(t+1)\), and all \(t\ge 2\),
\[
    \mathbb{E}\bigl[V_{t+1}\bigr] \le V_t - \frac{\eta}{4(t+1)}||\nabla f(x_t)||^2+\frac{\eta}{2(t+1)}\sigma^2
\]
which after telescoping yields
\[
    \min_{2\le t\le T-1}||\nabla f(x_t)||^2\le \frac{1}{T}\sum_{t=2}^{T-1}||\nabla f(x_t)||^2\le \frac{4V_2}{\eta \log T}+2\sigma^2.
\]
\end{thm}

\begin{proof}
By Lemma \ref{lem:pot-descent} we have:
\begin{equation}
\label{eq:baseline-pot-func-proof}
\begin{aligned}
\mathbb{E}\bigl[V_{t+1}\bigr] &\le V_t
-\frac{c_{t+1}\eta_t}{4}
  \|\nabla f(x_t)\|^2 + \frac{c_{t+1}\eta_t}{2}\sigma^2\\
&\quad+\Bigr(\frac{c_{t+1}}{2\eta_t}-\frac{c_t(1-L\eta_t c_t)}{2\eta_t(1-c_t)^2}
-\frac{3L^3c_{t+1}^2\eta_t^2}{2}(1-\beta_t)^2
+\frac{5L^2c_{t+1}\eta_t}{2}(1-\beta_t)^2\Bigr)||\delta_t||^2.
\end{aligned}
\end{equation}

Substituting \(c_{t+1}=1/(t+1)\) and constant \(\eta_t=\eta\) yields:
\begin{equation}
\label{eq:equation-a}
\begin{aligned}
\mathbb{E}\bigl[V_{t+1}\bigr] &\le V_t
-\frac{\eta}{4(t+1)}
  \|\nabla f(x_t)\|^2 + \frac{\eta}{2(t+1)}\sigma^2\\
&\quad+\Bigr(\frac{1}{2\eta(t+1)}-\frac{(1-\tfrac{L\eta}{t})}{2\eta t(\tfrac{t-1}{t})^2}
-\frac{3L^3\eta^2}{2(t+1)^2}(1-\beta_t)^2
+\frac{5L^2\eta}{2(t+1)}(1-\beta_t)^2\Bigr)||\delta_t||^2\\
&= V_t-\frac{\eta}{4(t+1)}\|\nabla f(x_t)\|^2 + \frac{\eta}{2(t+1)}\sigma^2\\
&\quad+\Bigl(\frac{(L\eta - 3)t+(L\eta + 1)}{2\eta(t+1)(t-1)^2}+\frac{L^2\eta}{2(t+1)}\Bigl(5-\frac{3L\eta}{t+1}\Bigr)(1-\beta_t)^2\Bigr)||\delta_t||^2\\
&=V_t-\frac{c_{t+1}\eta}{4}\|\nabla f(x_t)\|^2 + \frac{c_{t+1}\eta}{2}\sigma^2\\
&\quad+\Bigl(\frac{(L\eta - 3)t+(L\eta + 1)}{2\eta(t+1)(t-1)^2} + \frac{L^2\eta(1-\beta_t)^2}{2(t+1)^2}\bigl(5t + 5-3L\eta\bigr)\Bigr)||\delta_t||^2\\
&=V_t-\frac{c_{t+1}\eta}{4}\|\nabla f(x_t)\|^2 + \frac{c_{t+1}\eta}{2}\sigma^2\\
&\quad+\Bigl(\frac{(L\eta-3)t^2+(L\eta-2)t+(L\eta+1)}{2\eta(t+1)^2(t-1)^2}+\\
&\quad\quad+\frac{5L^2\eta^2t^2-3L^3\eta^3t+(3L^3\eta^3-5L^2\eta^2)}{2\eta(t+1)^2(t-1)^2}(1-\beta_t)^2\Bigr)||\delta_t||^2\\
&= V_t-\frac{c_{t+1}\eta}{4}\|\nabla f(x_t)\|^2 + \frac{c_{t+1}\eta}{2}\sigma^2\\
&\quad+\frac{1}{2\eta(t+1)^2(t-1)^2}\biggl((L\eta-3)t^2+(L\eta-2)t+(L\eta+1)+\\
&\quad\quad+\Bigl(5L^2\eta^2t^2-3L^3\eta^3t+(3L^3\eta^3-5L^2\eta^2)\Bigr)(1-\beta_t)^2\biggr)||\delta_t||^2\\
&= V_t -\frac{c_{t+1}\eta}{4}\|\nabla f(x_t)\|^2 
+\frac{c_{t+1}\eta}{2}\sigma^2\\
&\quad+\frac{1}{2\eta(t+1)^2(t-1)^2}\Bigl(\bigl[(L\eta-3)+5L^2\eta^2(1-\beta_t)^2\bigr]t^2\\
&\quad\quad+\bigl[(L\eta-2)-3L^3\eta^3(1-\beta_t)^2\bigr]t\\
&\quad\quad+\bigl[(L\eta+1)+(3L^3\eta^3-5L^2\eta^2)(1-\beta_t)^2\bigr]
\Bigr)\|\delta_t\|^2.
\end{aligned}
\end{equation}

We wish to show that, for some \(t^*\ge2\) the coefficient of \(||\delta_t||^2\) is negative for all \(t\ge t^*\). To do so, we simplify the negativity condition by considering the negativity of each \(t^2\), \(t\), and constant-ordered terms in the coefficient of the \(||\delta_t||^2\) term:
\[
\begin{aligned}
    &\bigl[(L\eta-3)+5L^2\eta^2(1-\beta_t)^2\bigr]t^2 \le 0\\
    \Rightarrow &(1-\beta_t)^2\le \frac{3-L\eta}{5L^2\eta^2}\\[15pt]
    &\bigl[(L\eta-2)-3L^3\eta^3(1-\beta_t)^2\bigr]t\le 0\\
    \Rightarrow&(1-\beta_t)^2\ge \frac{L\eta-2}{3L^3\eta^3},\quad \text{(which is always true for } L\eta\le 1)\\[15pt]
    &(L\eta+1)+(3L^3\eta^3-5L^2\eta^2)(1-\beta_t)^2\le0\\
    \Rightarrow&(1-\beta_t)^2\le\frac{-L\eta-1}{L^2\eta^2(3L\eta-5)}\\
    \Rightarrow&(1-\beta_t)^2\le\frac{L\eta+1}{L^2\eta^2(5-3L\eta)}
\end{aligned}
\]
Putting this altogether, for all \(t\ge t^*\), where \(t^*\) is defined as the first timestep for which
\[
    (1-\beta_{t})^2 \le \min\Bigl\{\frac{3-L\eta}{5L^2\eta^2},\frac{L\eta+1}{L^2\eta^2(5-3L\eta)}\Bigr\},
\]
we have that Equation \ref{eq:equation-a} simplifies to:
\begin{equation}
\label{eq:blah}
\begin{aligned}
&\mathbb{E}\bigl[V_{t+1}\bigr]\le V_t-\frac{\eta}{4(t+1)}
  \|\nabla f(x_t)\|^2 + \frac{\eta}{2(t+1)}\sigma^2\\
\Rightarrow&\frac{\eta}{4(t+1)}||\nabla f(x_t)||^2\le V_t-\mathbb{E}\bigl[V_{t+1}\bigr]+\frac{\eta}{2(t+1)}\sigma^2.\\
\end{aligned}
\end{equation}
Observe that for \(\beta_t=1\Rightarrow t^*=0\). For any \(0\le\beta_t<1\), our analysis suggests that neither the bias convergence rate nor the noise decay improves compared to only using \(\beta_t=1\). Therefore, for the sake of simplicity in completing the telescoping sum on Equation \ref{eq:blah}, take \(\beta_t=1\Rightarrow t^*=0\):
\[
\begin{aligned}
    &\frac{\eta}{4(t+1)}||\nabla f(x_t)||^2\le V_t-\mathbb{E}\bigl[V_{t+1}\bigr]+\frac{\eta}{2(t+1)}\sigma^2\quad\forall t\ge 2\\
    \Rightarrow&\frac{\eta}{4}\sum_{t=2}^{T-1}\frac{1}{t+1}||\nabla f(x_t)||^2\le V_2 + \frac{\sigma^2\eta}{2}\sum_{t=2}^{T-1}\frac{1}{t+1}\\
    \Rightarrow&\biggl(\frac{\eta}{4}\sum_{t=2}^{T-1}\frac{1}{t+1}\biggr)\min_{2\le t\le T-1}||\nabla f(x_t)||^2\le \frac{\eta\log T}{4}\min_{2\le t\le T-1}||\nabla f(x_t)||^2\le V_2 + \frac{\sigma^2\eta\log T}{2}\\
    \Rightarrow&\min_{2\le t\le T-1}||\nabla f(x_t)||^2\le \frac{4V_2}{\eta\log T}+2\sigma^2
\end{aligned}
\]
\end{proof}

\subsection{Theorem 4}
\begin{thm}
Define Schedule-Free by the three sequences:
\[
\begin{aligned}
        y_t &= (1 - \beta_t) z_t + \beta_t x_t,\\
        z_{t+1} &= z_t - \eta_t\nabla f(y_t,\zeta_t)\\
        x_{t+1}&=(1-c_{t+1}) x_t+c_{t+1} z_{t+1},
\end{aligned}
\]

and let \(V_t\) be defined as
\[
    V_t = f(x_t)+A_t||\delta_t||^2,\quad A_t=\frac{c_{t+1}(1-L\eta_t c_{t+1})}{2\eta_t(1-c_{t+1})^2},
\]
where Assumptions \ref{ass:assumption-1}, \ref{ass:assumption-2} and \ref{ass:assumption-3} hold. Then, for \(t\ge2\), \(\eta_t = \eta_0(t+1)\), \(\eta_0\le1/L\), and \(c_{t+1} = 1/(t+1)\),
\[
    \frac{\eta_0}{4}||\nabla f(x_t)||^2\le V_t - \mathbb{E}\bigl[V_{t+1}\bigr]+\frac{\eta_0}{2}\sigma^2+D^2\cdot O\Bigl(\frac{1}{t+1}\Bigr)
\]
which after telescoping yields
\[
    \min_{2\le t\le T-1}||\nabla f(x_t)||^2\le  \frac{4V_2}{\eta_0T}+D^2\cdot O\Bigl(\frac{\log T}{T}\Bigr) + 2\sigma^2.
\]
\end{thm}

\begin{proof}
By Lemma \ref{lem:pot-descent} we have:
\begin{equation}
\label{eq:linear-step-pot-func-proof}
\begin{aligned}
\mathbb{E}\bigl[V_{t+1}\bigr] &\le V_t
-\frac{c_{t+1}\eta_t}{4}
  \|\nabla f(x_t)\|^2 + \frac{c_{t+1}\eta_t}{2}\sigma^2\\
&\quad+\Bigr(\frac{c_{t+1}}{2\eta_t}-\frac{c_t(1-L\eta_t c_t)}{2\eta_t(1-c_t)^2}
-\frac{3L^3c_{t+1}^2\eta_t^2}{2}(1-\beta_t)^2
+\frac{5L^2c_{t+1}\eta_t}{2}(1-\beta_t)^2\Bigr)||\delta_t||^2.
\end{aligned}
\end{equation}

Substituting \(c_{t+1}=1/(t+1)\) and \(\eta_t=\eta_0(t+1)\) yields:
\begin{equation}
\label{eq:equation-b}
\begin{aligned}
\mathbb{E}\bigl[V_{t+1}\bigr] &\le V_t
-\frac{\eta_t}{4(t+1)}
  \|\nabla f(x_t)\|^2 + \frac{\eta_t}{2(t+1)}\sigma^2\\
&\quad+\Bigr(\frac{1}{2\eta_t(t+1)}-\frac{(1-\tfrac{L\eta_t}{t})}{2\eta_t t(\tfrac{t-1}{t})^2}
-\frac{3L^3\eta_t^2}{2(t+1)^2}(1-\beta_t)^2
+\frac{5L^2\eta_t}{2(t+1)}(1-\beta_t)^2\Bigr)||\delta_t||^2\\
&= V_t-\frac{\eta_t}{4(t+1)}\|\nabla f(x_t)\|^2 + \frac{\eta_t}{2(t+1)}\sigma^2\\
&\quad+\Bigl(\frac{(L\eta_t - 3)t+(L\eta_t + 1)}{2\eta_t(t+1)(t-1)^2}+\frac{L^2\eta_t}{2(t+1)}\Bigl(5-\frac{3L\eta_t}{t+1}\Bigr)(1-\beta_t)^2\Bigr)||\delta_t||^2\\
&=V_t-\frac{c_{t+1}\eta_t}{4}\|\nabla f(x_t)\|^2 + \frac{c_{t+1}\eta_t}{2}\sigma^2\\
&\quad+\Bigl(\frac{(L\eta_t - 3)t+(L\eta_t + 1)}{2\eta_t(t+1)(t-1)^2} + \frac{L^2\eta_t(1-\beta_t)^2}{2(t+1)^2}\bigl(5t + 5-3L\eta_t\bigr)\Bigr)||\delta_t||^2\\
&=V_t-\frac{c_{t+1}\eta_t}{4}\|\nabla f(x_t)\|^2 + \frac{c_{t+1}\eta_t}{2}\sigma^2\\
&\quad+\Bigl(\frac{(L\eta_t-3)t^2+(L\eta_t-2)t+(L\eta_t+1)}{2\eta_t(t+1)^2(t-1)^2}+\\
&\quad\quad+\frac{5L^2\eta_t^2t^2-3L^3\eta_t^3t+(3L^3\eta_t^3-5L^2\eta_t^2)}{2\eta_t(t+1)^2(t-1)^2}(1-\beta_t)^2\Bigr)||\delta_t||^2\\
&= V_t-\frac{c_{t+1}\eta_t}{4}\|\nabla f(x_t)\|^2 + \frac{c_{t+1}\eta_t}{2}\sigma^2\\
&\quad+\frac{1}{2\eta_t(t+1)^2(t-1)^2}\biggl((L\eta_t-3)t^2+(L\eta_t-2)t+(L\eta_t+1)+\\
&\quad\quad+\Bigl(5L^2\eta_t^2t^2-3L^3\eta_t^3t+(3L^3\eta_t^3-5L^2\eta_t^2)\Bigr)(1-\beta_t)^2\biggr)||\delta_t||^2\\
&= V_t -\frac{c_{t+1}\eta_t}{4}\|\nabla f(x_t)\|^2 
+\frac{c_{t+1}\eta_t}{2}\sigma^2\\
&\quad+\frac{1}{2\eta_t(t+1)^2(t-1)^2}\Bigl(\bigl[(L\eta_t-3)+5L^2\eta_t^2(1-\beta_t)^2\bigr]t^2\\
&\quad\quad+\bigl[(L\eta_t-2)-3L^3\eta_t^3(1-\beta_t)^2\bigr]t\\
&\quad\quad+\bigl[(L\eta_t+1)+(3L^3\eta_t^3-5L^2\eta_t^2)(1-\beta_t)^2\bigr]
\Bigr)\|\delta_t\|^2\\
&=V_t -\frac{\eta_0}{4}\|\nabla f(x_t)\|^2 
+\frac{\eta_0}{2}\sigma^2\\
&\quad+\frac{1}{2\eta_0(t+1)^3(t-1)^2}\Bigl(\bigl[(L\eta_0(t+1)-3)+5L^2\eta_0^2(t+1)^2(1-\beta_t)^2\bigr]t^2\\
&\quad\quad+\bigl[(L\eta_0 (t+1)-2)-3L^3\eta_0^3 (t+1)^3(1-\beta_t)^2\bigr]t\\
&\quad\quad+\bigl[(L\eta_0 (t+1)+1)+(3L^3\eta_0^3 (t+1)^3-5L^2\eta_0^2 (t+1)^2)(1-\beta_t)^2\bigr]
\Bigr)\|\delta_t\|^2\\
&=V_t -\frac{\eta_0}{4}\|\nabla f(x_t)\|^2 
+\frac{\eta_0}{2}\sigma^2\\
&\quad+\frac{1}{2\eta_0(t+1)^3(t-1)^2}\biggl[(t+1)\bigl(L\eta_0(t^2+t+1)-(3t-1)\bigr)\\
&\quad\quad+L^2\eta_0^2(5-3L\eta_0)\,(t-1)(t+1)^3\,(1-\beta_t)^2\biggr]||\delta_t||^2\\
&=V_t -\frac{\eta_0}{4}\|\nabla f(x_t)\|^2 
+\frac{\eta_0}{2}\sigma^2\\
&\quad + \biggl[\frac{L(t^2+t+1)-(3t-1)}{2(t+1)^2(t-1)^2}+\frac{L^2\eta_0(5-3L\eta_0)(1-\beta_t)^2}{2}\biggr]||\delta_t||^2.
\end{aligned}
\end{equation}

Consider the coefficient multiplying the \(||\delta_t||^2\) term:
\[
    E(t) = \frac{L(t^2+t+1)-(3t-1)}{2(t+1)^2(t-1)^2}+\frac{L^2\eta_0(5-3L\eta_0)(1-\beta_t)^2}{2(t+1)},
\]
which we want to minimize as much as possible. Clearly, \(\beta_t\) should equal one to cancel out the second term in \(E(t)\). This leaves us with:
\[
    E(t) = \frac{L(t^2+t+1)-(3t-1)}{2(t+1)^2(t-1)^2}=\frac{Lt^2+(L-3)t+(L+1)}{2(t+1)^2(t-1)^2}.
\]

Observe that \(E(t)\) has a positive leading coefficient and is therefore asymptotically positive. 
\[
\begin{aligned}
    \mathbb{E}\bigl[V_{t+1}\bigr] &\le V_t -\frac{\eta_0}{4}\|\nabla f(x_t)\|^2+\frac{\eta_0}{2}\sigma^2+ \frac{L(t^2+t+1)-(3t-1)}{2(t+1)^2(t-1)^2}||\delta_t||^2\\
    &\le V_t -\frac{\eta_0}{4}\|\nabla f(x_t)\|^2 +\frac{\eta_0}{2}\sigma^2+ O\Bigl(\frac{1}{(t+1)^2}\Bigr)||\delta_t||^2.
\end{aligned}
\]

Furthermore, by Assumption \ref{ass:assumption-2} (there exists some constant \(D\) such that \(||\delta_t||^2\le D^2(t+1)\)), we have:
\[
\begin{aligned}
    &\mathbb{E}\bigl[V_{t+1}\bigr]\le V_t -\frac{\eta_0}{4}\|\nabla f(x_t)\|^2 +\frac{\eta_0}{2}\sigma^2+ D^2\cdot O\Bigl(\frac{1}{t+1}\Bigr)\\
    \Rightarrow&\frac{\eta_0}{4}||\nabla f(x_t)||^2\le V_t - \mathbb{E}\bigl[V_{t+1}\bigr]+\frac{\eta_0}{2}\sigma^2+D^2\cdot O\Bigl(\frac{1}{t+1}\Bigr).
\end{aligned}
\]

Telescoping yields:
\[
\begin{aligned}
    &\frac{\eta_0}{4}\sum_{t=2}^{T-1}||\nabla f(x_t)||^2\le V_2 +\frac{\eta_0}{2}\sigma^2T+D^2\cdot O\Bigl(\log T\Bigr)\\
    \Rightarrow&\frac{\eta_0}{4}\sum_{t=2}^{T-1}\min_{2\le t\le T-1}||\nabla f(x_t)||^2 \le V_2 +\frac{\eta_0}{2}\sigma^2T+D^2\cdot O\Bigl(\log T\Bigr)\\
    \Rightarrow& \min_{2\le t\le T-1}||\nabla f(x_t)||^2\le \frac{4V_2}{\eta_0T}+D^2\cdot O\Bigl(\frac{\log T}{T}\Bigr) + 2\sigma^2,
\end{aligned}
\]
which is the desired rate.

\end{proof}

\subsection{Theorem 5}
\begin{thm}
Define Schedule-Free by the three sequences:
\[
\begin{aligned}
        y_t &= (1 - \beta_t) z_t + \beta_t x_t,\\
        z_{t+1} &= z_t - \eta_t\nabla f(y_t,\zeta_t)\\
        x_{t+1}&=(1-c_{t+1}) x_t+c_{t+1} z_{t+1},
\end{aligned}
\]

and let \(V_t\) be defined as
\[
    V_t = f(x_t)+A_t||\delta_t||^2,\quad A_t=\frac{c_{t+1}(1-L\eta_t c_{t+1})}{2\eta_t(1-c_{t+1})^2},
\]
where Assumptions \ref{ass:assumption-1} and \ref{ass:assumption-3} hold. Then, for \(t\ge2\), constant \(\eta \le 1/L\) \(\beta_t=1\), and \(c_{t+1} = 1/(t+1)^\alpha\),
\[
    \mathbb{E}\bigl[V_{t+1}\bigr] \le V_t -\frac{\eta}{4(t+1)^\alpha}\|\nabla f(x_t)\|^2+\frac{\eta}{2(t+1)^\alpha}\sigma^2
\]
which after telescoping yields
\[
    \min_{2\le t \le T-1}\|\nabla f(x_t)\|^2\le \begin{cases}
        \frac{4V_2\bigl(1-\alpha\bigr)}{\eta\bigl(T^{1-\alpha}-2^{1-\alpha}\bigr)}+2\sigma^2, & \alpha\in[0,1)\\
        \frac{4V_2}{\eta\log T} + 2\sigma^2, & \alpha = 1\\
        O(1),&\alpha > 1
    \end{cases}
\]
\end{thm}

\begin{proof}\hfill \\ \\
\textbf{Decreasing c-Averaging (}\(\mathbf{c_{t+1}=1/(t+1)^\alpha):}\) By Lemma \ref{lem:pot-descent} we have:
\begin{equation}
\label{eq:dec-c--pot-func-proof}
\begin{aligned}
\mathbb{E}\bigl[V_{t+1}\bigr] &\le V_t
-\frac{c_{t+1}\eta_t}{4}
  \|\nabla f(x_t)\|^2 + \frac{c_{t+1}\eta_t}{2}\sigma^2\\
&\quad+\Bigr(\frac{c_{t+1}}{2\eta_t}-\frac{c_t(1-L\eta_t c_t)}{2\eta_t(1-c_t)^2}
-\frac{3L^3c_{t+1}^2\eta_t^2}{2}(1-\beta_t)^2
+\frac{5L^2c_{t+1}\eta_t}{2}(1-\beta_t)^2\Bigr)||\delta_t||^2.
\end{aligned}
\end{equation}

Substituting \(c_{t+1}=1/(t+1)^\alpha\) and constant \(\eta_t=\eta\) yields:
\begin{equation}
\label{eq:equation-c}
\begin{aligned}
\mathbb{E}[V_{t+1}]
&\le V_t-\frac{\,\eta}{4(t+1)^\alpha}\|\nabla f(x_t)\|^2+\frac{\,\eta}{2(t+1)^\alpha}\,\sigma^2\\
&\quad+\Biggl(\frac{1}{2\eta\,(t+1)^\alpha}-\frac{\displaystyle \frac{1}{t^\alpha}\bigl(1-\tfrac{L\eta}{t^\alpha}\bigr)}{\displaystyle 2\eta\,(1-\tfrac{1}{t^\alpha})^2}\frac{3L^3\eta^2}{2\,(t+1)^{2\alpha}}\,(1-\beta_t)^2+\frac{5L^2\eta}{2\,(t+1)^\alpha}\,(1-\beta_t)^2\Biggr)\|\delta_t\|^2\\[6pt]
&=V_t-\frac{\eta}{4\,(t+1)^\alpha}\|\nabla f(x_t)\|^2+\frac{\eta}{2\,(t+1)^\alpha}\sigma^2\\
&\quad+\Biggl(\frac{1}{2\eta\,(t+1)^\alpha}-\frac{t^\alpha - L\eta}{2\eta\,(t^\alpha - 1)^2}-\frac{3L^3\eta^2}{2\,(t+1)^{2\alpha}}(1-\beta_t)^2+\frac{5L^2\eta}{2\,(t+1)^\alpha}(1-\beta_t)^2\Biggr)\|\delta_t\|^2\\
&=V_t-\frac{\eta}{4\,(t+1)^\alpha}\|\nabla f(x_t)\|^2+\frac{\eta}{2\,(t+1)^\alpha}\sigma^2\\
&\quad+\Bigl(\frac{(t^\alpha-1)^2-(t+1)^\alpha (t^\alpha -L\eta)}{2\eta(t+1)^\alpha (t^\alpha -1)^2} + \frac{L^2\eta\bigl(5t+5-3L\eta\bigr)}{2(t+1)^{2\alpha}}(1-\beta_t)^2\biggr)||\delta_t||^2.
\end{aligned}
\end{equation}

Similarly to our previous proofs, we again see that \(\beta_t=1\) minimizes the coefficient of the \(||\delta_t||^2\) term. This choice of \(\beta_t\) yields the following simplification on Equation \ref{eq:equation-c}:
\begin{equation}
\label{eq:equation-d}
\begin{aligned}
\mathbb{E}[V_{t+1}]
&\le V_t-\frac{\eta}{4\,(t+1)^\alpha}\|\nabla f(x_t)\|^2+\frac{\eta}{2\,(t+1)^\alpha}\sigma^2+\frac{(t^\alpha-1)^2-(t+1)^\alpha (t^\alpha -L\eta)}{2\eta(t+1)^\alpha (t^\alpha -1)^2}||\delta_t||^2.
\end{aligned}
\end{equation}

Now, we analyze when the coefficient of the \(||\delta_t||^2\) term to determine when it is non-positive. Observe that for all \(t\ge2\),
\[
\begin{aligned}
\frac{(t^\alpha-1)^2 - (t+1)^\alpha\,(t^\alpha - L\eta)}
     {2\eta\,(t+1)^\alpha\,(t^\alpha -1)^2}
\le0
&\Longleftrightarrow
(t+1)^\alpha\,(t^\alpha - L\eta) \ge (t^\alpha -1)^2\\
&\Longleftrightarrow (t+1)^\alpha \ge \frac{(t^\alpha -1)^2}{(t^\alpha - L\eta)}\ge \frac{t^{2\alpha}-2t^\alpha + 1}{t^\alpha} = t^\alpha -2+\frac{1}{t^\alpha},
\end{aligned}
\]
which is always true for \(t\ge2\). Therefore, we have that
\begin{equation}
\label{eq:equation-e}
\begin{aligned}
    &\mathbb{E}\bigl[V_{t+1}\bigr] \le V_t -\frac{\eta}{4(t+1)^\alpha}\|\nabla f(x_t)\|^2+\frac{\eta}{2(t+1)^\alpha}\sigma^2,\quad\forall t\ge 2,\;\alpha\ge 0\\
    \Rightarrow&\frac{\eta}{4(t+1)^\alpha}\|\nabla f(x_t)\|^2\le V_t-\mathbb{E}\bigl[V_{t+1}\bigr]+\frac{\eta}{2(t+1)^\alpha}\sigma^2.
\end{aligned}
\end{equation}

Telescoping Equation \ref{eq:equation-e}, we get:
\[
\begin{aligned}
    &\frac{\eta}{4}\sum_{t=2}^{T-1}\frac{1}{(t+1)^\alpha}\|\nabla f(x_t)\|^2\le V_2+\frac{\eta}{2}\sum_{t=2}^{T-1}\frac{1}{(t+1)^\alpha}\sigma^2\\
    \Rightarrow&\Bigl(\frac{\eta}{4}\sum_{t=2}^{T-1}\frac{1}{(t+1)^\alpha}\Bigr)\min_{2\le t \le T-1}\|\nabla f(x_t)\|^2\le V_2+\frac{\eta}{2}\sum_{t=2}^{T-1}\frac{1}{(t+1)^\alpha}\sigma^2\\
    \Rightarrow&\Bigl(\sum_{t=2}^{T-1}\frac{1}{(t+1)^\alpha}\Bigr)\min_{2\le t \le T-1}\|\nabla f(x_t)\|^2\le \frac{4V_2}{\eta}+2\sum_{t=2}^{T-1}\frac{1}{(t+1)^\alpha}\sigma^2.
\end{aligned}
\]
Now, consider the three possible cases for \(\alpha\):
\begin{itemize}
    \item \(\mathbf{\alpha\in[0, 1):}\) \(\sum_{t=2}^{T-1}\frac{1}{(t+1)^\alpha} \ge \frac{T^{1-\alpha}-2^{1-\alpha}}{1-\alpha}\)
    \item \(\mathbf{\alpha=1:}\) \(\sum_{t=2}^{T-1}\frac{1}{(t+1)^\alpha} \ge \log T\)
    \item \(\mathbf{\alpha>1:}\) \(\sum_{t=2}^{T-1}\frac{1}{(t+1)^\alpha} \ge \text{constant}\)
\end{itemize}

Thus, we have
\[
    \min_{2\le t \le T-1}\|\nabla f(x_t)\|^2\le \begin{cases}
        \frac{4V_2\bigl(1-\alpha\bigr)}{\eta\bigl(T^{1-\alpha}-2^{1-\alpha}\bigr)}+2\sigma^2, & \alpha\in[0,1)\\
        \frac{4V_2}{\eta\log T} + 2\sigma^2, & \alpha = 1\\
        O(1),&\alpha > 1
    \end{cases}
\]
which is the desired result for \(c_{t+1}=1/(t+1)^\alpha\).\\

\textbf{Increasing c-Averaging (}\(\mathbf{c_{t+1}=t^\alpha/(t+1)^\alpha):}\) Again, by Lemma \ref{lem:pot-descent} we have:
\begin{equation}
\label{eq:inc-c--pot-func-proof}
\begin{aligned}
\mathbb{E}\bigl[V_{t+1}\bigr] &\le V_t
-\frac{c_{t+1}\eta_t}{4}
  \|\nabla f(x_t)\|^2 + \frac{c_{t+1}\eta_t}{2}\sigma^2\\
&\quad+\Bigr(\frac{c_{t+1}}{2\eta_t}-\frac{c_t(1-L\eta_t c_t)}{2\eta_t(1-c_t)^2}
-\frac{3L^3c_{t+1}^2\eta_t^2}{2}(1-\beta_t)^2
+\frac{5L^2c_{t+1}\eta_t}{2}(1-\beta_t)^2\Bigr)||\delta_t||^2.
\end{aligned}
\end{equation}

Substituting \(c_{t+1}=t^\alpha/(t+1)^\alpha\) and constant \(\eta_t=\eta\) yields:
\begin{equation}
\label{eq:equation-f}
\begin{aligned}
\mathbb{E}[V_{t+1}]
&\le V_t
-\frac{t^\alpha\eta}{4\,(t+1)^\alpha}\,\|\nabla f(x_t)\|^2
+\frac{t^\alpha\eta}{2\,(t+1)^\alpha}\,\sigma^2\\
&\quad+\Biggl(
\frac{t^\alpha}{2\eta\,(t+1)^\alpha}
-\frac{\displaystyle\frac{(t-1)^\alpha}{t^\alpha}
       \bigl(1-\tfrac{L\eta\,(t-1)^\alpha}{t^\alpha}\bigr)}
     {\displaystyle2\eta\bigl(1-\tfrac{(t-1)^\alpha}{t^\alpha}\bigr)^2}
-\frac{3L^3\,\eta^2\,t^{2\alpha}}{2\,(t+1)^{2\alpha}}(1-\beta_t)^2
+\frac{5L^2\,\eta\,t^\alpha}{2\,(t+1)^\alpha}(1-\beta_t)^2
\Biggr)\|\delta_t\|^2\\
&=V_t
-\frac{t^\alpha\eta}{4\,(t+1)^\alpha}\,\|\nabla f(x_t)\|^2
+\frac{t^\alpha\eta}{2\,(t+1)^\alpha}\,\sigma^2\\
&\quad+\Biggl(\frac{t^\alpha}{2\eta\,(t+1)^\alpha}-\frac{(t-1)^\alpha\bigl(t^\alpha - L\eta(t-1)^\alpha\bigr)}{2\eta t^\alpha\bigl(t^\alpha-(t-1)^\alpha\bigr)^2} +L^2\eta\cdot\frac{5L t^\alpha(t+1)^\alpha-3L\eta t^{2\alpha}}{2(t+1)^{2\alpha}}(1-\beta_t)^2\Biggr)||\delta_t||^2\\
&=V_t
-\frac{t^\alpha\eta}{4\,(t+1)^\alpha}\,\|\nabla f(x_t)\|^2
+\frac{t^\alpha\eta}{2\,(t+1)^\alpha}\,\sigma^2\\
&\quad+\Biggl(\frac{t^{2\alpha}(t^\alpha -(t-1)^\alpha)^2-(t+1)^\alpha(t-1)^\alpha\bigl(t^\alpha - L\eta(t-1)^\alpha\bigr)}{2\eta t^\alpha(t+1)^\alpha(t^\alpha -(t-1)^\alpha)^2}\\
&\quad\quad+L^2\eta\cdot\frac{5L t^\alpha(t+1)^\alpha-3L\eta t^{2\alpha}}{2(t+1)^{2\alpha}}(1-\beta_t)^2\Biggr)||\delta_t||^2\\
\end{aligned}
\end{equation}

We choose \(\beta_t=1\) to help minimize the coefficient of the \(||\delta_t||^2\) term. This gives us:
\begin{equation}
\label{eq:equation-g}
\begin{aligned}
\mathbb{E}[V_{t+1}]
&\le V_t
-\frac{t^\alpha\eta}{4\,(t+1)^\alpha}\,\|\nabla f(x_t)\|^2
+\frac{t^\alpha\eta}{2\,(t+1)^\alpha}\,\sigma^2\\
&\quad+\frac{t^{2\alpha}(t^\alpha -(t-1)^\alpha)^2-(t+1)^\alpha(t-1)^\alpha\bigl(t^\alpha - L\eta(t-1)^\alpha\bigr)}{2\eta t^\alpha(t+1)^\alpha(t^\alpha -(t-1)^\alpha)^2}||\delta_t||^2\\
&\le V_t
-\frac{t^\alpha\eta}{4\,(t+1)^\alpha}\,\|\nabla f(x_t)\|^2
+\frac{t^\alpha\eta}{2\,(t+1)^\alpha}\,\sigma^2\\
&\quad+\frac{t^{2\alpha}(t^\alpha -(t-1)^\alpha)^2-(t+1)^\alpha(t-1)^\alpha\bigl(t^\alpha - (t-1)^\alpha\bigr)}{2\eta t^\alpha(t+1)^\alpha(t^\alpha -(t-1)^\alpha)^2}||\delta_t||^2\\
\end{aligned}
\end{equation}

where, in the last step, we used \(L\eta\le1\). To analyze the non-positivity of the remaining \(||\delta_t||^2\) coefficient, observe that we have the following two cases for \(\alpha\):
\begin{itemize}
    \item \(\mathbf{\alpha\in[0,1):}\) \[
\begin{aligned}
&t^{2\alpha}\bigl(t^\alpha-(t-1)^\alpha\bigr)^2
-(t+1)^\alpha(t-1)^\alpha\bigl(t^\alpha-(t-1)^\alpha\bigr)\\
&=\;(t^\alpha-(t-1)^\alpha)
\Bigl[t^{2\alpha}\bigl(t^\alpha-(t-1)^\alpha\bigr)-(t+1)^\alpha(t-1)^\alpha\Bigr]
\end{aligned}
\]
Since \(x\mapsto x^\alpha\) is concave on \(\alpha\in[0,1)\), by Bernoulli’s inequality we have, \[
\begin{aligned}
(t-1)^\alpha &\ge t^\alpha - \alpha\,t^{\alpha-1},
\quad (t+1)^\alpha(t-1)^\alpha \ge t^{2\alpha}\Bigl(1-\tfrac{\alpha}{t^2}\Bigr),
\end{aligned}
\]
\[
\begin{aligned}
\implies\;&t^{2\alpha}\bigl(t^\alpha-(t-1)^\alpha\bigr)
\;\le\;\alpha\,t^{3\alpha-1},\\
\implies\;&t^{2\alpha}\bigl(t^\alpha-(t-1)^\alpha\bigr)-(t+1)^\alpha(t-1)^\alpha
\;\le\;\alpha\,t^{3\alpha-1}-t^{2\alpha}\Bigl(1-\tfrac{\alpha}{t^2}\Bigr)\\
&\;=\;t^{2\alpha-1}\bigl[\alpha\,t^\alpha-(t^2-\alpha)\bigr]
\;<\;0\quad(\forall\,t\ge2,\;0\le\alpha<1).
\end{aligned}
\]
Hence the entire numerator is negative under the stated assumptions.

    \item \(\mathbf{\alpha\ge1:}\) \[
\begin{aligned}
N
&=t^{2\alpha}\bigl(t^\alpha-(t-1)^\alpha\bigr)^2 
\;-\;(t+1)^\alpha(t-1)^\alpha\bigl(t^\alpha - (t-1)^\alpha\bigr),\\
\Delta&:=t^\alpha-(t-1)^\alpha>0.
\end{aligned}
\]
By convexity of \(x^\alpha\) for \(\alpha\ge1\) and Bernoulli’s inequality,
\[
\Delta\ge \alpha\,t^{\alpha-1},\quad
(t+1)^\alpha\le t^\alpha+\alpha\,t^{\alpha-1},\quad
(t-1)^\alpha\ge t^\alpha-\alpha\,t^{\alpha-1}.
\]
Hence
\[
t^{2\alpha}\Delta^2\;\ge\;t^{2\alpha}(\alpha\,t^{\alpha-1})^2
=\alpha^2\,t^{4\alpha-2},
\quad
(t+1)^\alpha(t-1)^\alpha
\le t^{2\alpha}-\alpha^2\,t^{2\alpha-2}.
\]
It follows that
\[
N\;\ge\;\alpha^2\,t^{4\alpha-2}-\bigl(t^{2\alpha}-\alpha^2\,t^{2\alpha-2}\bigr)
=t^{2\alpha-2}\Bigl(\alpha^2\,t^{2\alpha}-(t^2-\alpha^2)\Bigr)
=t^{2\alpha-2}\bigl[\alpha\,t^{\alpha+1}-t^2+\alpha^2\bigr].
\]
But for \(\alpha\ge1\) and \(t\ge1\) we have \(t^{\alpha+1}\ge t^2\) and \(\alpha^2>0\), so
\[
\alpha\,t^{\alpha+1}-t^2+\alpha^2
\ge t^2-t^2+\alpha^2
=\alpha^2>0,
\]
and hence \(N>0\).  Dividing by the positive denominator shows the full coefficient is positive for all 
        \(\alpha\ge1\), \(t\ge1\).

\end{itemize}

Altogether, we have for \(\alpha\in[0,1)\) and \(L\eta\le1\),
\[
\begin{aligned}
    &\mathbb{E}[V_{t+1}]\le V_t
-\frac{t^\alpha\eta}{4\,(t+1)^\alpha}\,\|\nabla f(x_t)\|^2
+\frac{t^\alpha\eta}{2\,(t+1)^\alpha}\,\sigma^2\\
\Rightarrow&\frac{t^\alpha\eta}{4\,(t+1)^\alpha}\,\|\nabla f(x_t)\|^2\le V_t-\mathbb{E}\bigl[V_{t+1}\bigr]+\frac{t^\alpha\eta}{2\,(t+1)^\alpha}\,\sigma^2.
\end{aligned}
\]
Telescoping this Equation yields the following:\[
\begin{aligned}
\biggl(\frac{\eta}{4}\sum_{t=2}^{T-1}\frac{t^\alpha}{(t+1)^\alpha}\biggr)\min_{2\le t\le T-1}\|\nabla f(x_t)\|^2\le V_2+\frac{\sigma^2\eta}{2}\sum_{t=2}^{T-1}\frac{t^\alpha}{(t+1)^\alpha}.
\end{aligned}
\]
Observing that \begin{align*}
\sum_{t=2}^{T-1}\Bigl(\frac{t}{t+1}\Bigr)^\alpha
&=\sum_{k=3}^{T}\Bigl(1-\frac{1}{k}\Bigr)^\alpha
=\sum_{k=3}^{T}\exp\Bigl(\alpha\ln(1-\tfrac1k)\Bigr)\\
&=\sum_{k=3}^{T}\Bigl(1-\tfrac{a}{k}+O(k^{-2})\Bigr)
=(T-2)-\alpha\sum_{k=3}^{T}\frac{1}{k}+O(1)\\
&=(T-2)-\alpha\Bigl(H_T-\tfrac{3}{2}\Bigr)+O(1)
=(T-2)-\alpha\log T+O(1)\\
&\ge T-2-\alpha \log T \quad(\alpha\in[0,1)),
\end{align*}
it follows that
\[
\begin{aligned}
&\frac{\eta}{4}\bigl(T-2-\alpha\log T\bigr)\min_{2\le t\le T-1}\|\nabla f(x_t)\|^2\le V_2+\frac{\sigma^2\eta}{2}\bigl(T-2-\alpha\log T\bigr)\\
\Rightarrow&\min_{2\le t\le T-1}\|\nabla f(x_t)\|^2\le \frac{4V_2}{\eta T-2\eta -\alpha\eta\log T} + 2\sigma^2.
\end{aligned}
\]
This completes the proof.
\end{proof}

\clearpage
\section{Full SDP Formulation for PEP}
\subsection{Decreasing Stepsize Case ($c_k = 1/(t+1)^\alpha$)}
\[
\begin{aligned}
\max_{G,\delta} \quad & \|g_n\|^2 \\
\text{s.t.} \quad & \text{Tr}(G^T A_{i,j} G) \leq f_i - f_j \quad \forall i < j = 0,\dots,n \\
& \text{Tr}(G^T C_k G) \leq 0 \quad \text{for } k=0,\dots,n-1 \quad \text{(update constraints)} \\
& f_0 - f_n \leq D \\
& G = \begin{bmatrix} g_0 & \cdots & g_n \end{bmatrix}^T \in \mathbb{R}^{(n+1)\times d} \\
& \delta = [f_0 \cdots f_n]^T \in \mathbb{R}^{n+1}
\end{aligned}
\]

where:
\begin{itemize}
\item $g_i = \nabla f(x_i)/L$ (normalized gradients)
\item $A_{i,j}$ encodes $L$-smoothness:
\[
A_{i,j} = e_i(e_i-e_j)^T + \tfrac{1}{2}\|e_i-e_j\|^2I
\]
\item $C_k$ encodes the update rules:
\[
C_k = \begin{cases} 
\text{Constraints for } y_t = (1-\beta)z_t + \beta x_t \\
\text{Constraints for } z_{t+1} = z_t - \eta g(y_t) \\
\text{Constraints for } x_{t+1} = (1-c_t)x_k]t + c_t z_{t+1}
\end{cases}
\]
\end{itemize}

\subsection{Increasing Stepsize Case ($c_{t+1} = (t/(t+1))^\alpha$)}
Identical structure to the decreasing case, but with modified $C_k$ matrices reflecting the different stepsize rule.

\subsection{Distance Case ($\eta_t = (t+1)/L$)}
\[
\begin{aligned}
\max_{G,X,\delta} \quad & \|x_n - z_n\|^2 \\
\text{s.t.} \quad & \text{Tr}(G^T A_{i,j} G) \leq f_i - f_j \quad \forall i < j \\
& \text{Tr}([G|X]^T D_k [G|X]) \leq 0 \quad \text{(distance dynamics)} \\
& f_0 - f_n \leq D \\
& G = \begin{bmatrix} g_0 & \cdots & g_n \end{bmatrix}^T \in \mathbb{R}^{(n+1)\times d} \\
& X = \begin{bmatrix} x_0 & \cdots & x_n \end{bmatrix}^T \in \mathbb{R}^{(n+1)\times d}
\end{aligned}
\]

where $D_k$ encodes:
\begin{itemize}
\item The linear stepsize rule: $z_{t+1} = z_t - \tfrac{t+1}{L}g(y_t)$
\item The averaging: $x_{t+1} = \left(1-\tfrac{1}{t+1}\right)x_k + \tfrac{1}{t+1}z_{t+1}$
\end{itemize}

\clearpage
\section{Experimental Setup and Code}
Using the SDP formulations presented in Appendix E, we applied the Python 3.8 PEPit library to validate our analysis of Schedule-Free under the various choices of hyperparameters we considered. \\ \\
For each choice of hyperparameter, we ran PEP for \(n=100\) iterations, with \(\beta_t=1\) (see proofs in earlier Appendices), sweeping values of \(\alpha=\{0.01,\,0.1\,0.5,\,1.0\}\). We also used PEP to empirically justify Assumption \ref{ass:assumption-2} by running the SDP with \(c_{t+1}=1/(t+1),\,\eta_t=\eta_0(t+1),\,\eta_0=1/L\) for \(n=100\) steps. \\ \\ 
Listing 2 contains the code used to solve each of these SDPs and Listing 3 contains the code for generating the Figures presented throughout the paper.

\subsection{PEPit SDP Solver Code}
\begin{lstlisting}[language=Python, caption={PEPit SDP Solver Code}]
# ------------------------------------------------------------------
# Environment guards
# ------------------------------------------------------------------
import os, multiprocessing as mp
os.environ["OMP_NUM_THREADS"]       = "1"
os.environ["MKL_NUM_THREADS"]       = "1"
os.environ["OPENBLAS_NUM_THREADS"]  = "1"

# -----------------------------------------------------------------
# Imports
# -----------------------------------------------------------------
import itertools, pickle, sys, logging
import numpy as np
from concurrent.futures import ProcessPoolExecutor, as_completed
from PEPit import PEP
from PEPit.functions import SmoothFunction

# ------------------------------------------------------------------
# Logging
# ------------------------------------------------------------------
logging.basicConfig(
    level=logging.INFO,
    format="[%(asctime)s] pid%(process)d %(message)s",
    datefmt="%Y-%m-%d %H:%M:%S",
)
log = logging.getLogger(__name__)

# ------------------------------------------------------------------
# Global parameters
# ------------------------------------------------------------------
L, gamma, beta, D = 1.0, 1.0, 1.0, 1.0
ALPHAS = (0.01, 0.1, 0.5, 1.0)
N_STEPS = np.arange(1, 101)
OUTFILE = "pep_results_combined.pkl"

MAX_WORKERS = min(32, os.cpu_count())
CTX = mp.get_context("spawn")

# ------------------------------------------------------------------
# Worker functions
# ------------------------------------------------------------------
def dec_worker(alpha_n):
    try:
        alpha, n = alpha_n
        problem = PEP()
        f = problem.declare_function(SmoothFunction, L=L)
        x0 = problem.set_initial_point()
        x = z = x0
        _, f0 = f.oracle(x0)

        for k in range(n):
            y = (1 - beta) * z + beta * x
            gy, _ = f.oracle(y)
            z -= gamma * gy
            c_k = 1 / (k + 1) ** alpha
            x = (1 - c_k) * x + c_k * z
            gx, _ = f.oracle(x)
            problem.set_performance_metric(gx ** 2)

        problem.set_initial_condition((f0 - f.oracle(x)[1]) <= D)
        tau = problem.solve(wrapper="cvxpy", verbose=0)
        return ("dec", alpha, n, tau)

    except Exception:
        import traceback
        traceback.print_exc()
        sys.stdout.flush(); sys.stderr.flush()
        raise


def inc_worker(alpha_n):
    try:
        alpha, n = alpha_n
        problem = PEP()
        f = problem.declare_function(SmoothFunction, L=L)
        x0 = problem.set_initial_point()
        x = z = x0
        _, f0 = f.oracle(x0)

        for k in range(n):
            y = (1 - beta) * z + beta * x
            gy, _ = f.oracle(y)
            z -= gamma * gy
            c_k = (k / (k + 1)) ** alpha
            x = (1 - c_k) * x + c_k * z
            gx, _ = f.oracle(x)
            problem.set_performance_metric(gx ** 2)

        problem.set_initial_condition((f0 - f.oracle(x)[1]) <= D)
        tau = problem.solve(wrapper="cvxpy", verbose=0)
        return ("inc", alpha, n, tau)

    except Exception:
        import traceback
        traceback.print_exc()
        sys.stdout.flush(); sys.stderr.flush()
        raise


def dist_worker(n):
    try:
        problem = PEP()
        f = problem.declare_function(SmoothFunction, L=L)
        x0 = problem.set_initial_point()
        x = z = x0
        _, f0 = f.oracle(x0)

        for k in range(n):
            y = (1 - beta) * z + beta * x
            gy, _ = f.oracle(y)
            z -= gamma * (k + 1) * gy
            x = (1 - 1 / (k + 1)) * x + (1 / (k + 1)) * z
            gx, _ = f.oracle(x)
            problem.set_performance_metric((x - z) ** 2)

        problem.set_initial_condition((f0 - f.oracle(x)[1]) <= D)
        tau = problem.solve(wrapper="cvxpy", verbose=0)
        return ("dist", None, n, tau)

    except Exception:
        import traceback
        traceback.print_exc()
        sys.stdout.flush(); sys.stderr.flush()
        raise


def lin_step_worker(n):
    try:
        problem = PEP()
        f = problem.declare_function(SmoothFunction, L=L)
        x0 = problem.set_initial_point()
        x = z = x0
        _, f0 = f.oracle(x0)

        for k in range(n):
            y = (1 - beta) * z + beta * x
            gy, _ = f.oracle(y)
            z -= gamma * (k + 1) * gy
            x = (1 - 1/(k+1)) * x + (1/(k+1)) * z
            gx, _ = f.oracle(x)
            problem.set_performance_metric(gx ** 2)

        problem.set_initial_condition((f0 - f.oracle(x)[1]) <= D)
        tau = problem.solve(wrapper="cvxpy", verbose=0)
        return ("lin_step", None, n, tau)

    except Exception:
        import traceback
        traceback.print_exc()
        sys.stdout.flush(); sys.stderr.flush()
        raise

# ------------------------------------------------------------------
# Driver
# ------------------------------------------------------------------
def main():
    if os.path.exists(OUTFILE):
        log.error("%s already exists - aborting.", OUTFILE)
        return

    # Initialize results structure
    results = {
        "decreasing": {a: np.empty_like(N_STEPS, dtype=float) for a in ALPHAS},
        "increasing": {a: np.empty_like(N_STEPS, dtype=float) for a in ALPHAS},
        "distance": np.empty_like(N_STEPS, dtype=float),
        "linear_grad": np.empty_like(N_STEPS, dtype=float),
        "n_steps": N_STEPS
    }

    # Prepare task lists
    tasks_dec = list(itertools.product(ALPHAS, N_STEPS))
    tasks_inc = list(itertools.product(ALPHAS, N_STEPS))
    tasks_dist = list(N_STEPS)
    tasks_lingrad = list(N_STEPS)
    
    TOTAL = len(tasks_dec) + len(tasks_inc) + len(tasks_dist) + len(tasks_lingrad)
    done = 0

    log.info("Launching %d total tasks with %d worker processes...", TOTAL, MAX_WORKERS)

    with ProcessPoolExecutor(max_workers=MAX_WORKERS,
                           mp_context=CTX) as pool:

        # Submit all tasks
        futures = (
            [pool.submit(dec_worker, t) for t in tasks_dec] +
            [pool.submit(inc_worker, t) for t in tasks_inc] +
            [pool.submit(dist_worker, n) for n in tasks_dist] +
            [pool.submit(lin_step_worker, n) for n in tasks_lingrad]
        )

        # Process results as they complete
        for fut in as_completed(futures):
            kind, alpha, n, tau = fut.result()
            idx = n - 1  # convert to 0-based index
            
            if kind == "dec":
                results["decreasing"][alpha][idx] = tau
            elif kind == "inc":
                results["increasing"][alpha][idx] = tau
            elif kind == "dist":
                results["distance"][idx] = tau
            elif kind == "lin_step":
                results["linear_grad"][idx] = tau

            done += 1
            log.info("%4d/%d finished (%s alpha=%s, n=%d)", done, TOTAL, kind, alpha, n)

    # Save results
    with open(OUTFILE, "wb") as fh:
        pickle.dump(results, fh)
    log.info("All %d tasks done - results saved to %s", TOTAL, OUTFILE)

# ------------------------------------------------------------------
if __name__ == "__main__":
    main()
\end{lstlisting}

\subsection{Data Visualization}
\begin{lstlisting}[language=Python, caption={Data Visualization}]
# -----------------------------------------------------------------
# Imports
# -----------------------------------------------------------------
import pickle, numpy as np, pandas as pd, seaborn as sns, matplotlib.pyplot as plt
from matplotlib import rc

# -----------------------------------------------------------------
# Global plotting style
# -----------------------------------------------------------------
rc('font', family='serif', serif='Times')
rc('text', usetex=False)
plt.style.use("seaborn-v0_8-whitegrid")

FIGSIZE = (6, 4)
DPI     = 300
SAVE_KW = dict(bbox_inches='tight', dpi=DPI)

# -----------------------------------------------------------------
# Load results from PEP output
# -----------------------------------------------------------------
PEP_RESULTS_PATH = "pep_results.pkl" 

# Helper formulas
def dec_formula(t, alpha):
    return (t)**(1 - alpha)

def inc_formula(t, alpha):
    return t - 2 - alpha * np.log(t)

def lin_step_formula(t):
    return (t + 1) / np.log(t + 1)

# Load the results
with open(PEP_RESULTS_PATH, "rb") as f:
    results = pickle.load(f)

n_steps = results['n_steps']

# -----------------------------------------------------------------
# Create tidy dataframe for seaborn
# -----------------------------------------------------------------
frames = []
for case in ['decreasing', 'increasing']:
    for alpha in results[case].keys():
        for t in n_steps:
            if t < 3 or (alpha == 1 and case == 'increasing'):
                continue
            tau = results[case][alpha][t-1]
            weight = dec_formula(t, alpha) if case == 'decreasing'  \ else inc_formula(t, alpha)
            frames.append(dict(
                Case   = case.title(),
                Alpha  = rf'$\alpha={alpha}$',
                Step   = t,
                Weighted = tau * weight
            ))
df = pd.DataFrame(frames)

# -----------------------------------------------------------------
# 4.  Generate all plots
# -----------------------------------------------------------------
# Assumption 2: Distance between x_t and z_t
df_dist = pd.DataFrame({'Step': n_steps,
                       'Value': results['distance']})

plt.figure(figsize=FIGSIZE, dpi=DPI)
sns.lineplot(data=df_dist, x='Step', y='Value', marker='o')
plt.title(r"$\eta_t=(t+1)/L,\,c_{t+1}=1/(t+1))$")
plt.xlabel(r"Iteration $t$")
plt.ylabel(r"Upper Bound of $\|x_t - z_t\|^2$")
plt.savefig("distance.pdf", **SAVE_KW)
plt.show()

# Thm 4: Linear stepsize
weighted_grad = [g * lin_step_formula(t)
                for t, g in zip(n_steps, results['linear_grad'])]
df_wgrad = pd.DataFrame({'Step': n_steps,
                        'Value': weighted_grad})

plt.figure(figsize=FIGSIZE, dpi=DPI)
sns.lineplot(data=df_wgrad, x='Step', y='Value', marker='s')
plt.title(r"$\eta_t=(t+1)/L,\,c_{t+1}=1/(t+1))$")
plt.xlabel(r"Iteration $t$")
plt.ylabel(r"Upper Bound of $\|\nabla f(x_t)\|^2\,\frac{t}{\log t}$")
plt.savefig("linear_steps.pdf", **SAVE_KW)
plt.show()

# Thm 5: Decreasing stepsizes
plt.figure(figsize=FIGSIZE, dpi=DPI)
ax = sns.lineplot(data=df[df.Case == 'Decreasing'],
             x='Step', y='Weighted', hue='Alpha',
             style='Alpha', markers=True, dashes=False)
ax.legend(loc='upper right')
ax.set_ylim(bottom=0.0, top=2.0)
plt.title(r"$\eta_t\text{ constant},\,c_{t+1}=1/(t+1)^\alpha$")
plt.xlabel(r"Iteration $t$")
plt.ylabel(r"Upper Bound of $\|\nabla f(x_t)\|^2\,t^{1-\alpha}$")
plt.tight_layout()
plt.savefig("decreasing_steps.pdf", **SAVE_KW)
plt.show()

# Thm 5: Increasing stepsizes
plt.figure(figsize=FIGSIZE, dpi=DPI)
sns.lineplot(data=df[df.Case == 'Increasing'],
             x='Step', y='Weighted', hue='Alpha',
             style='Alpha', markers=True, dashes=False)
plt.title(r"$\eta_t\text{ constant},\,c_{t+1}=t^\alpha/(t+1)^\alpha$")
plt.xlabel(r"Iteration $t$")
plt.ylabel(r"Upper Bound of $\|\nabla f(x_t)\|^2\,(t-2-\alpha\log t)$")
plt.savefig("increasing_steps.pdf", **SAVE_KW)
plt.show()
\end{lstlisting}
\end{document}